\documentclass[letterpaper]{article} %
\usepackage{aaai23}  %

\usepackage{xr}

\usepackage{times}  %
\usepackage{helvet}  %
\usepackage{courier}  %
\usepackage[hyphens]{url}  %
\usepackage{graphicx} %
\usepackage{natbib}  %
\usepackage{caption} %
\usepackage{algorithm}
\usepackage{algorithmic}

\usepackage{amsfonts}
\usepackage{amsmath}
\usepackage{amsthm}
\usepackage{mathtools}
\usepackage{multirow}
\usepackage{booktabs}
\usepackage[dvipsnames]{xcolor}
\definecolor{blue_gp}{rgb}{0.122 0.467 0.706}
\definecolor{orange_mc_drop}{rgb}{1.00 0.498 0.055}
\definecolor{green_pne}{rgb}{0.173 0.627 0.173}
\definecolor{brown_nflows}{rgb}{0.549, 0.337, 0.294}
\definecolor{red_nflows_base}{rgb}{0.839 0.153 0.157}
\definecolor{purple_nflows_out}{rgb}{0.58 0.404 0.741}
\DeclarePairedDelimiterX{\infdivx}[2]{(}{)}{%
  #1\;\delimsize\|\;#2%
}

\usepackage{newfloat}
\usepackage{listings}
\usepackage{tikz}
\usepackage{xr}
\floatstyle{ruled}
\newfloat{listing}{tb}{lst}{}
\floatname{listing}{Listing}
\usepackage{hyperref}
\makeatletter
\newcommand*{\addFileDependency}[1]{
  \typeout{(#1)}
  \@addtofilelist{#1}
  \IfFileExists{#1}{}{\typeout{No file #1.}}
}
\makeatother

\newcommand*{\myexternaldocument}[1]{
    \externaldocument{#1}
    \addFileDependency{#1.tex}
    \addFileDependency{#1.aux}
}
\myexternaldocument{supp}
\title{Normalizing Flow Ensembles for Rich Aleatoric and Epistemic Uncertainty Modeling}
\author {
    Lucas Berry,\textsuperscript{\rm 1}
    David Meger \textsuperscript{\rm 1}
}
\affiliations {
    \textsuperscript{\rm 1} McGill University, Centre for Intelligent Machines, Montreal, Canada\\
    lucas.berry@mail.mcgill.ca
}
\theoremstyle{plain}
\newtheorem{theorem}{Theorem}[section]

\theoremstyle{definition}

\theoremstyle{remark}

\begin{document}

\maketitle

\begin{abstract}
 In this work, we demonstrate how to reliably estimate epistemic uncertainty while maintaining the flexibility needed to capture complicated aleatoric distributions. To this end, we propose an ensemble of Normalizing Flows (NF), which are state-of-the-art in modeling aleatoric uncertainty. The ensembles are created via sets of fixed dropout masks, making them less expensive than creating separate NF models. We demonstrate how to leverage the unique structure of NFs, base distributions, to estimate aleatoric uncertainty without relying on samples, provide a comprehensive set of baselines, and derive unbiased estimates for differential entropy. The methods were applied to a variety of experiments, commonly used to benchmark aleatoric and epistemic uncertainty estimation: 1D sinusoidal data, 2D windy grid-world (\textit{Wet Chicken}), \textit{Pendulum}, and \textit{Hopper}. In these experiments, we setup an active learning framework and evaluate each model's capability at measuring aleatoric and epistemic uncertainty. The results show the advantages of using NF ensembles in capturing complicated aleatoric while maintaining accurate epistemic uncertainty estimates.
\end{abstract}
\section{Introduction}
One common decomposition of uncertainty is aleatoric and epistemic \cite{hora1996aleatory,der2009aleatory, hullermeier2021aleatoric}. Aleatoric uncertainty refers to the inherent randomness in the outcome of an experiment, while epistemic uncertainty can be described as ignorance or a lack of knowledge. The important distinction between the two is that epistemic uncertainty can be reduced by the acquisition of more data while aleatoric cannot. Our goal in this paper is to learn  aleatoric distributions in high dimensions, with arbitrary distributional form, while also tracking epistemic uncertainty due to non-uniform data sampling.

Normalizing Flows (NFs) have been shown to be effective at capturing highly expressive aleatoric uncertainty with little prior knowledge \cite{kingma2018glow, nf-rezende15}. This is done by transforming a base distribution via a series of nonlinear bijective mappings, and can model complex heteroscedastic and multi-modal noise. Robotic systems display such noise, as robots interact with nonlinear stochastic dynamics often in high dimensions. To further complicate the problem of dynamics modeling, data collection for many robot systems can be prohibitively expensive. Therefore, an active learning framework is usually adopted to iteratively collect data in order to most efficiently improve a model. In order to apply such a framework, ensembles have been employed to capture epistemic uncertainty for deep learning models \cite{gal2017deep}.   

In this paper, we utilize NFs' ability to capture rich aleatoric uncertainty and extend such models to epistemic uncertainty estimation with ensembles. We then use our NF models to tackle epistemic uncertainty for regression tasks. The contributions of this work are as follows:
\begin{itemize}
    \item We develop two methods for estimating uncertainty for NFs, derive unbiased estimates for said models, and leverage the base distribution to reduce the sampling burden on the estimation of uncertainty.
    \item We leverage memory-efficient ensembling by creating ensembles via fixed dropout masks and apply them to NFs.
    \item We demonstrate the usefulness of uncertainty estimation on an array of previously proposed benchmarks \cite{depeweg2018decomposition} and novel settings for active learning problems.
\end{itemize}
\section{Problem Statement}
Given a dataset $\mathcal{D}=\{x_i, y_i\}_{i=1}^N$, where $x_i\in\mathbb{R}^K$ and $y_i\in\mathbb{R}^D$, we wish to approximate the conditional probability $p_{Y|X}(y|x)$. Our models therefore take as input $x$ and output the approximate conditional probability density function, $p_{Y|X}(y|x)=f_{\theta}(y,x)$, where $\theta$ is a set of parameters to be learned. In the experiments that follow, the ground truth distribution, $p_{Y|X}(y|x)$, is not assumed homoscedastic nor do we put restrictions on its shape. 

In order to capture uncertainty, we utilize the information based criteria proposed in \citet{houlsby2011bayesian} to estimate uncertainty. Let $H(\cdot)$ denote the differential entropy of a random variable, $H(X)= E\left[-\log(p_X(x))\right]$, and $W \sim p(w)$ index different models in a Bayesian framework or an ensemble. Then one can define epistemic uncertainty, $I(y^*, W)$, as the difference of total, $H(y^*|x^*)$, and aleatoric, $E_{p(w)}\left[H(y^*|x^*,w)\right]$, uncertainty,
\begin{align}
    I(y^*, W) &= H(y^*|x^*)-E_{p(w)}\left[H(y^*|x^*,w)\right] \label{eq:epi},
\end{align}
where $I(\cdot)$ is the mutual information (MI) and $y^*$, $x^*$ denote values not seen during training. 

These uncertainty estimates can be used for the problem of active learning, where we iteratively add data points to our training dataset in order to improve our model's performance as much as possible with each data point selected \cite{settles2009active, mackay1992information}.  When doing active learning, one attempts to find the $x^*$ that maximizes Equation (\ref{eq:epi}) and query new points in that region. Intuitively, we are looking for the point $x^*$ which has high total uncertainty, $H(y^*|x^*)$, but each model has low uncertainty, $H(y^*|x^*,w)$.

\section{Background} \label{sec:background}
NFs were chosen to capture aleatoric uncertainty as they are capable of representing flexible distributions. Ensembles were chosen as they are a computationally efficient way of estimating epistemic uncertainty for deep learning. 
\begin{figure*}[t]
\centering
\begin{tikzpicture}[scale=0.85]
\node[rectangle,draw] (input) at (0,0) {Input};
\node[inner sep=0pt] (base0) at (4,3)
    {\includegraphics[width=.30\columnwidth]{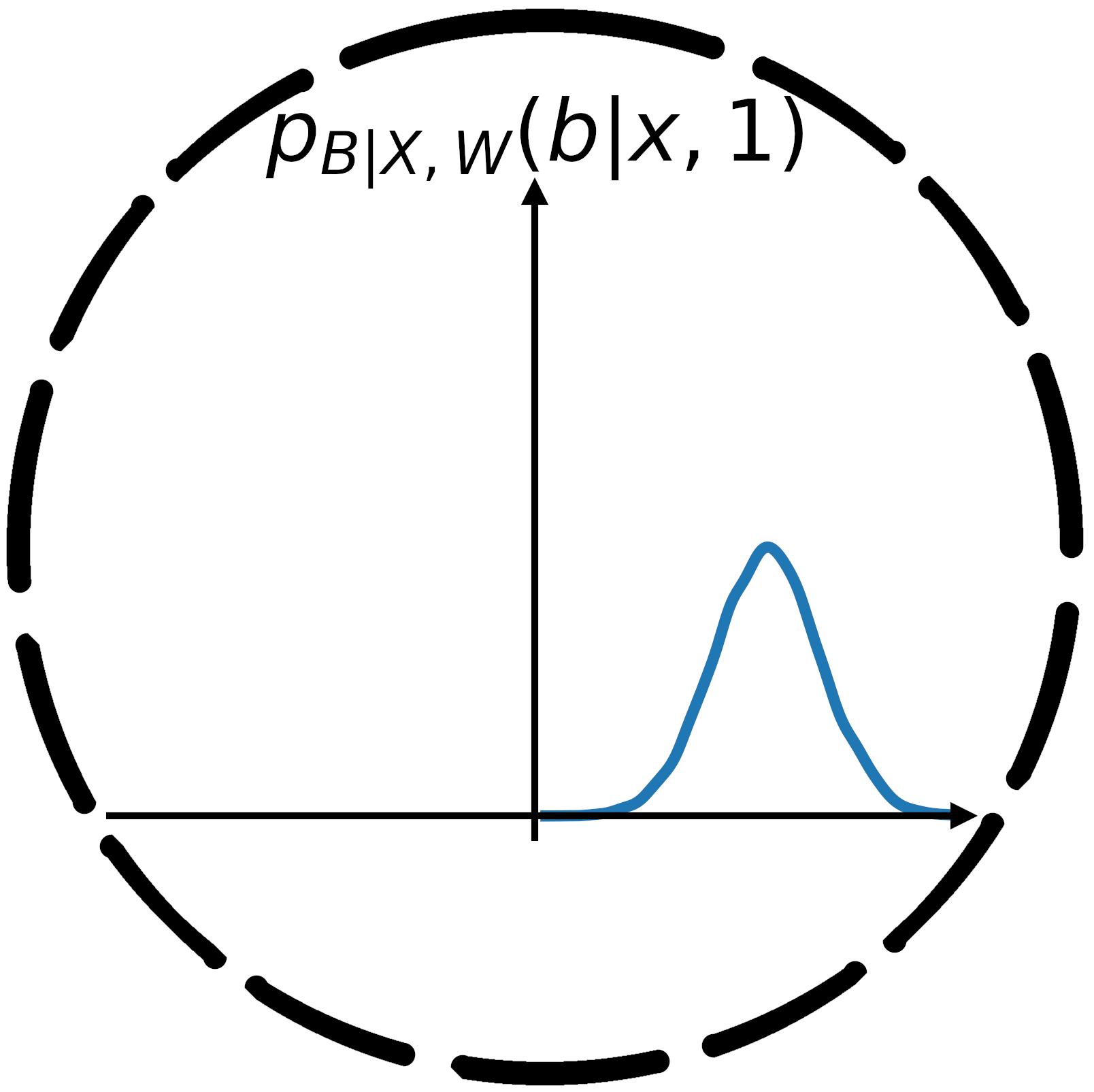}};
\node[inner sep=0pt] (base1) at (4,0)
    {\includegraphics[width=.30\columnwidth]{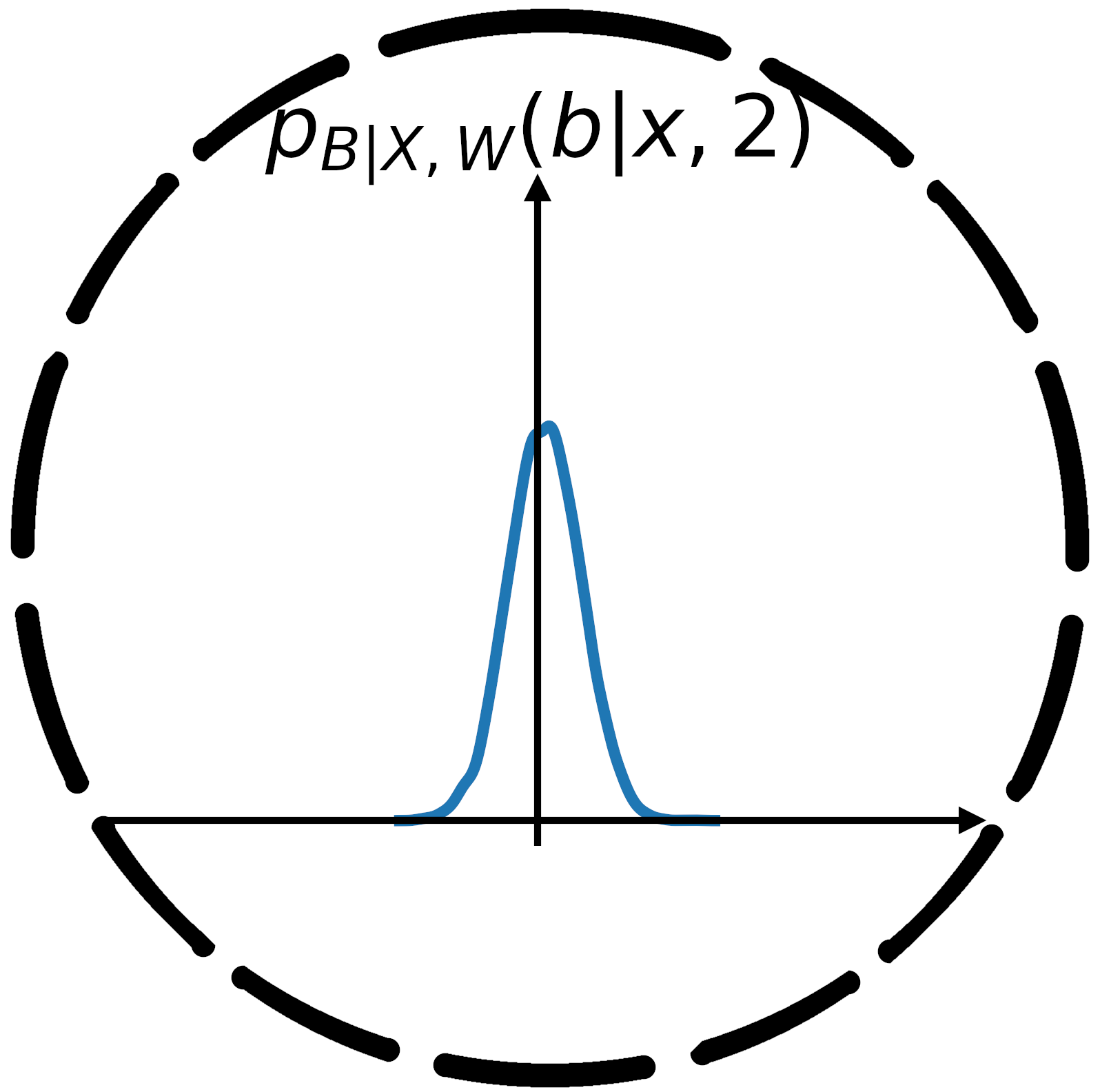}};
\node[inner sep=0pt] (basen) at (4,-3)
    {\includegraphics[width=.30\columnwidth]{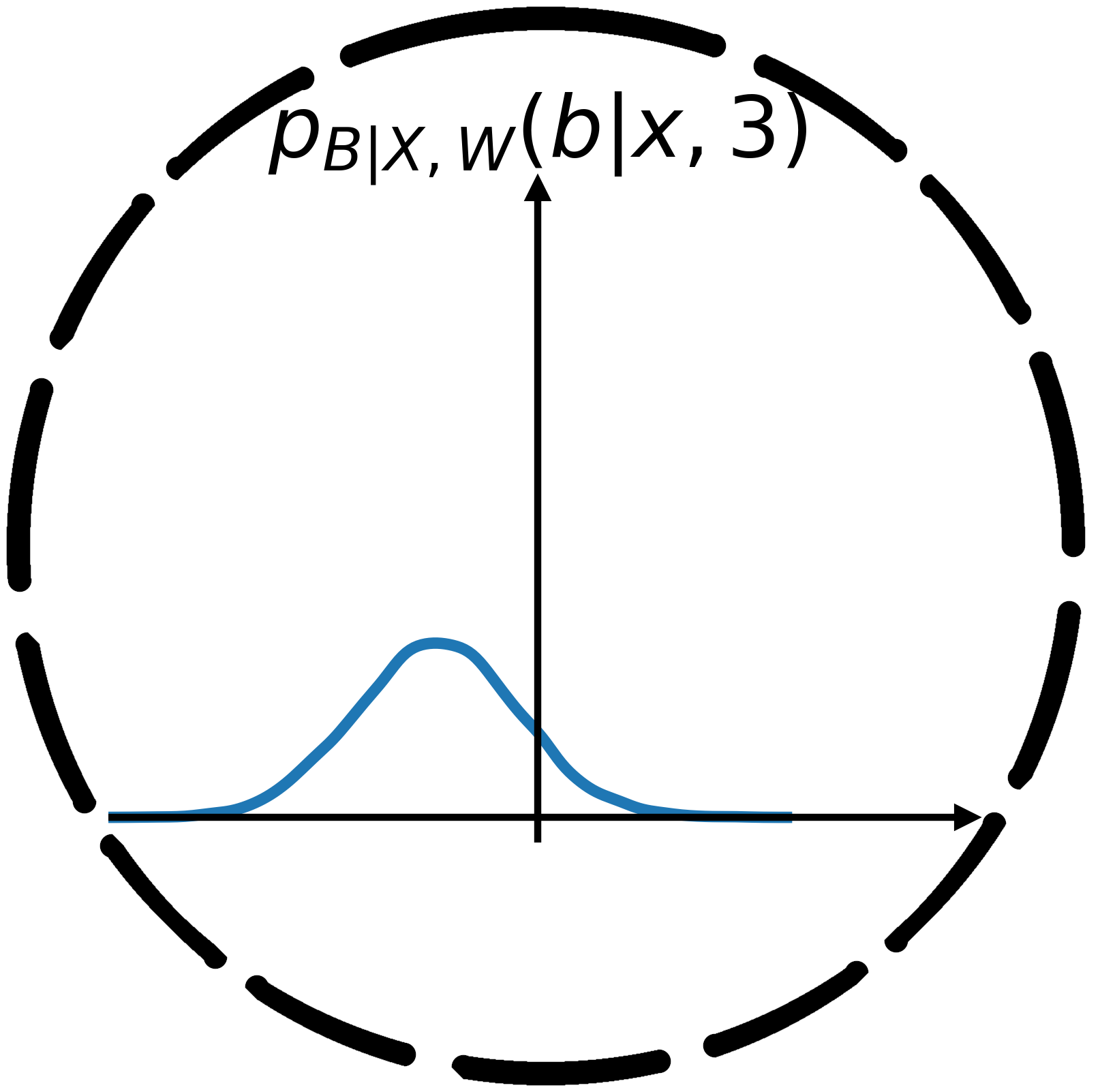}};
\draw[->,thick] (input) --  node [above, rotate=35]  {$N(\mu_1,\Sigma_1)$}(base0);
\draw[->,thick] (input) -- node [above] {$N(\mu_2,\Sigma_2)$}(base1);
\draw[->,thick] (input) -- node [above, rotate=-35]  {$N(\mu_3,\Sigma_3)$}(basen);

\node[inner sep=0pt] (out0) at (8,3)
    {\includegraphics[width=.30\columnwidth]{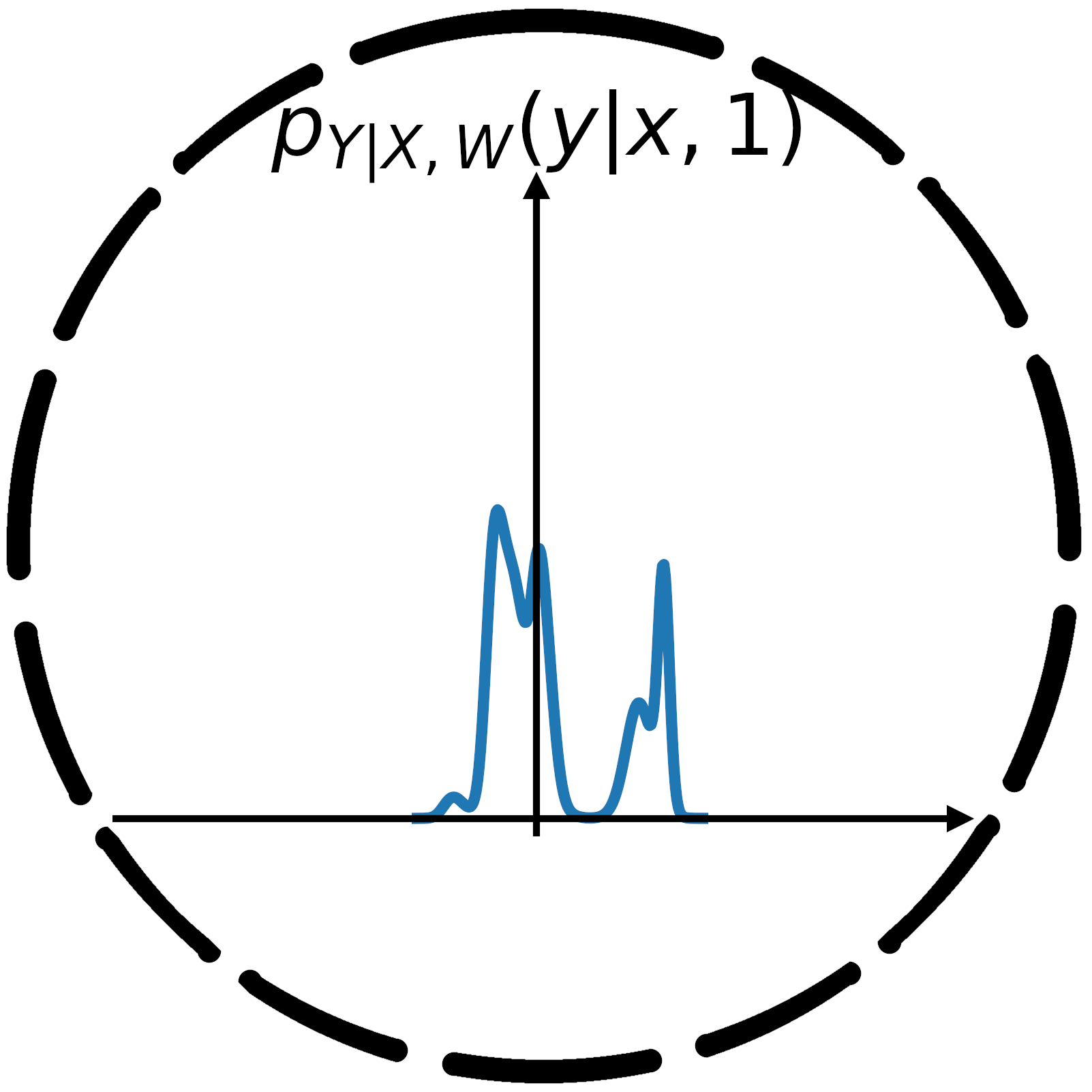}};
\node[inner sep=0pt] (out1) at (8,0)
    {\includegraphics[width=.30\columnwidth]{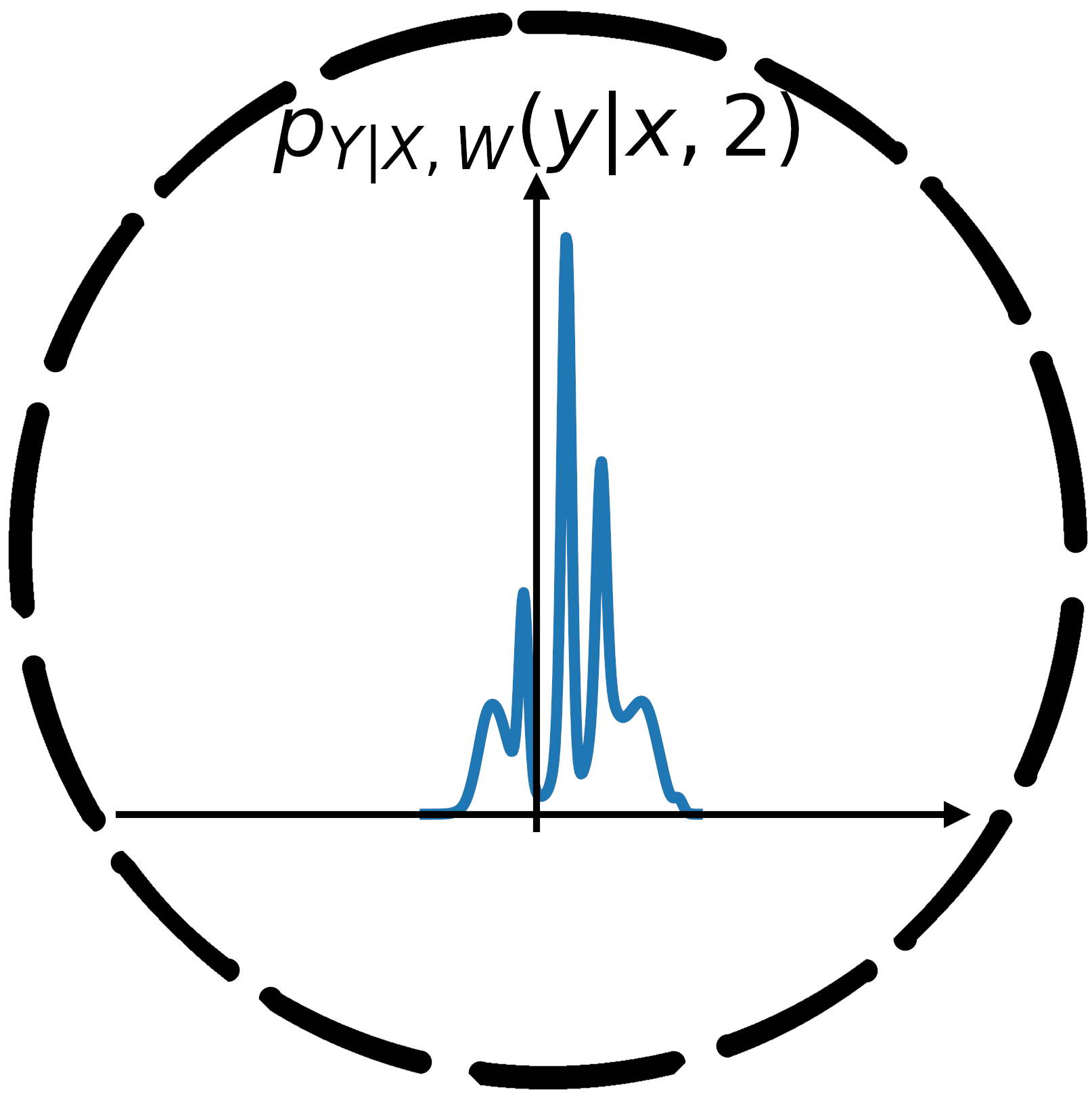}};
\node[inner sep=0pt] (outn) at (8,-3)
    {\includegraphics[width=.30\columnwidth]{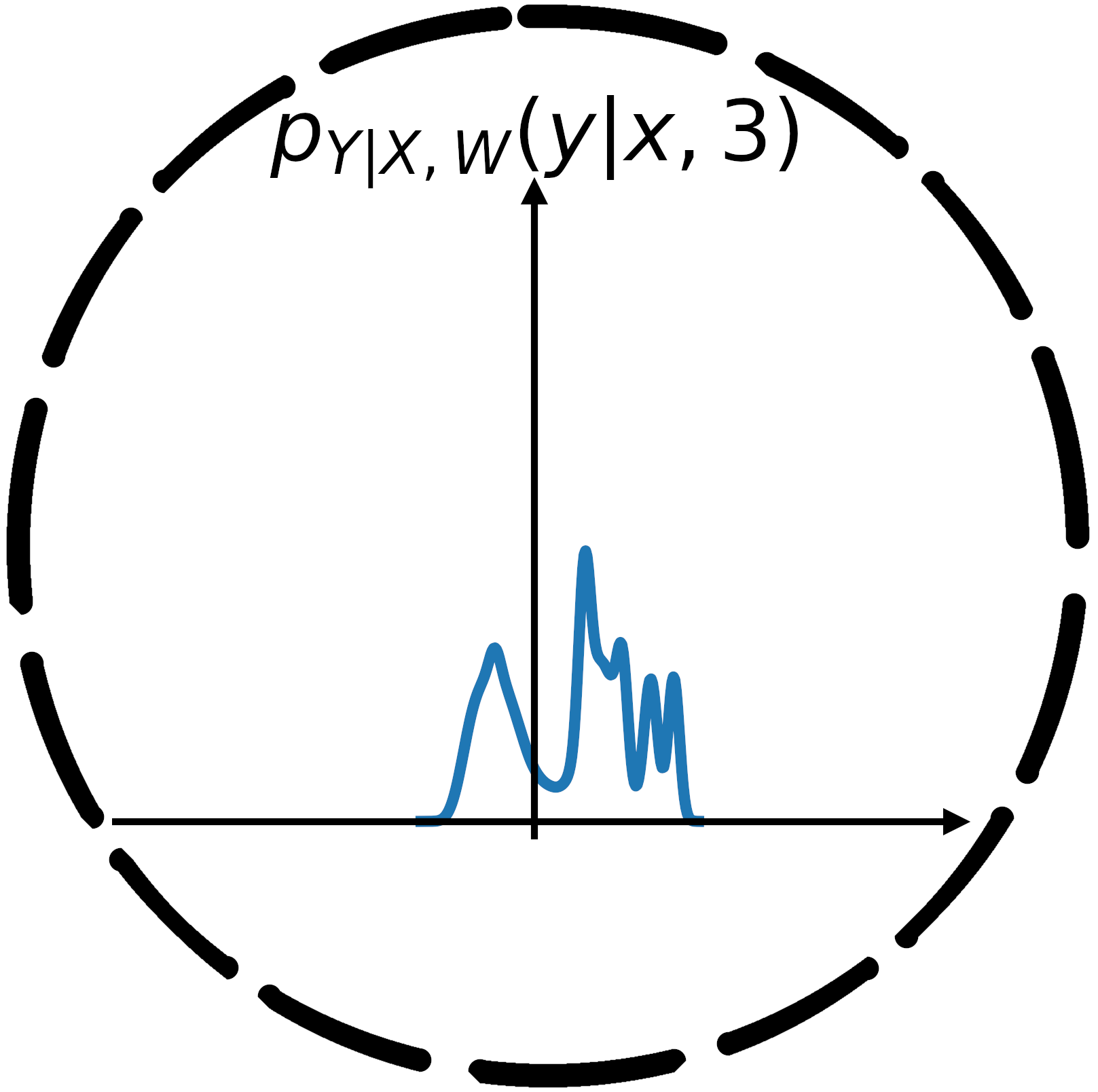}};
\draw[->,thick] (base0) -- node [above]  {$g_{\theta}$}(out0);
\draw[->,thick] (base1) -- node [above]  {$g_{\theta}$}(out1);
\draw[->,thick] (basen) -- node [above]  {$g_{\theta}$}(outn);
\node[rectangle,fill=white] (input) at (4,-5) {Base Distribution};
\node[rectangle,fill=white] (input) at (8,-5) {Output Distribution};
\end{tikzpicture}
\begin{tikzpicture}[scale=0.85]
\node[rectangle,draw] (input) at (0,0) {Input};
\node[inner sep=0pt] (base0) at (4,0)
    {\includegraphics[width=.30\columnwidth]{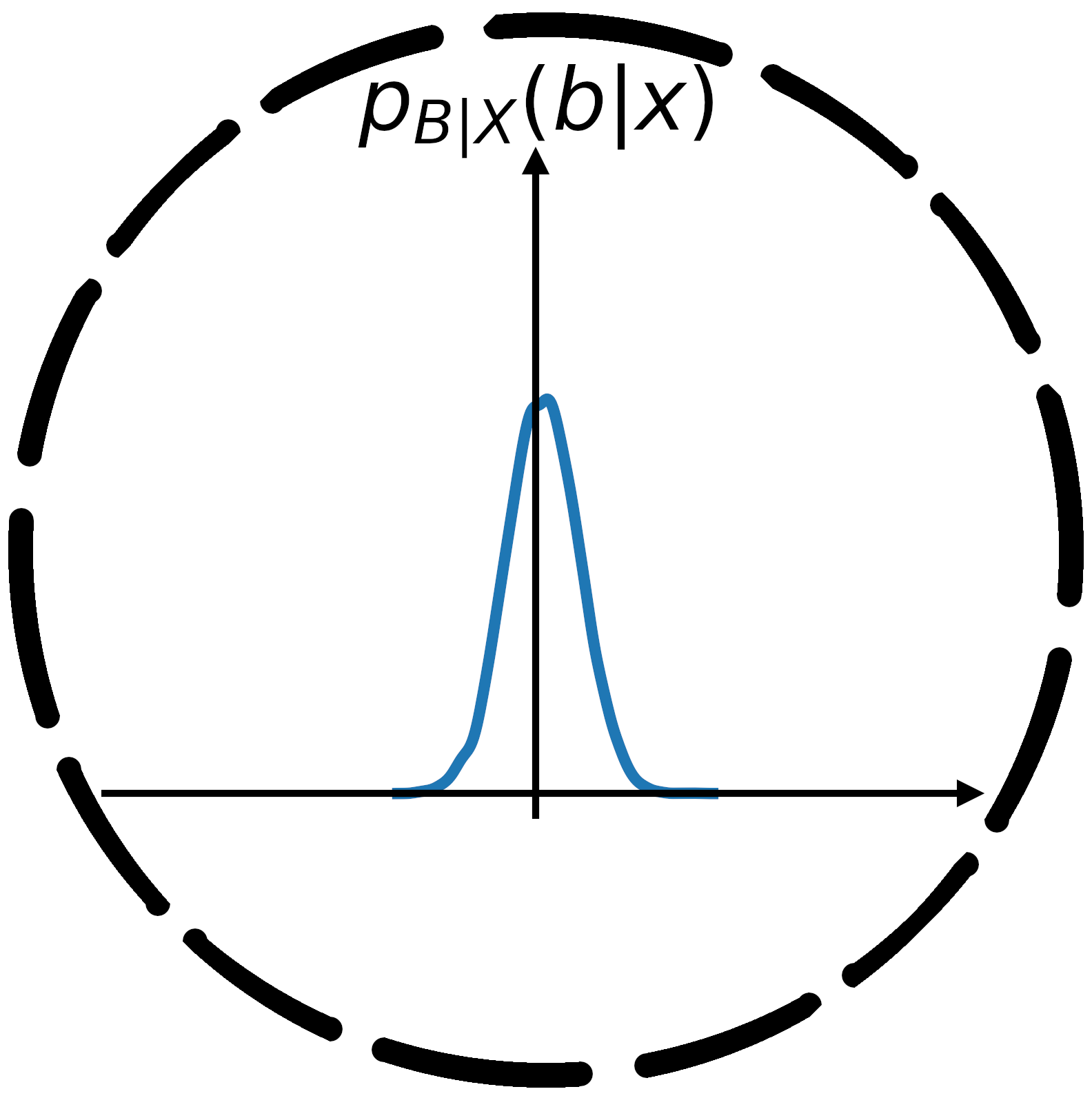}};
\draw[->,thick] (input) -- node [above] {$N(\mu,\Sigma)$}(base0);
\node[inner sep=0pt] (out0) at (8,3)
    {\includegraphics[width=.30\columnwidth]{figures/main_paper/out_0.png}};
\node[inner sep=0pt] (out1) at (8,0)
    {\includegraphics[width=.30\columnwidth]{figures/main_paper/out_1.png}};
\node[inner sep=0pt] (outn) at (8,-3)
    {\includegraphics[width=.30\columnwidth]{figures/main_paper/out_n.png}};
\draw[->,thick] (base0) --  node [above, rotate=35]  {$g_{\theta_1}$}(out0);
\draw[->,thick] (base0) --  node [above]  {$g_{\theta_2}$}(out1);
\draw[->,thick] (base0) --  node [above, rotate=-35]  {$g_{\theta_3}$}(outn);
\node[rectangle,fill=white] (input) at (4,-5) {Base Distribution};
\node[rectangle,fill=white] (input) at (8,-5) {Output Distribution};
\end{tikzpicture}
\caption{The left figure illustrates Nflows Base and the right figure depicts Nflows Out on an ensemble of 3 components with one bijective transformation.}
\label{fig:nflow_models}
\end{figure*}
\begin{figure*}[t]
\centering
\includegraphics[width=0.9\textwidth]{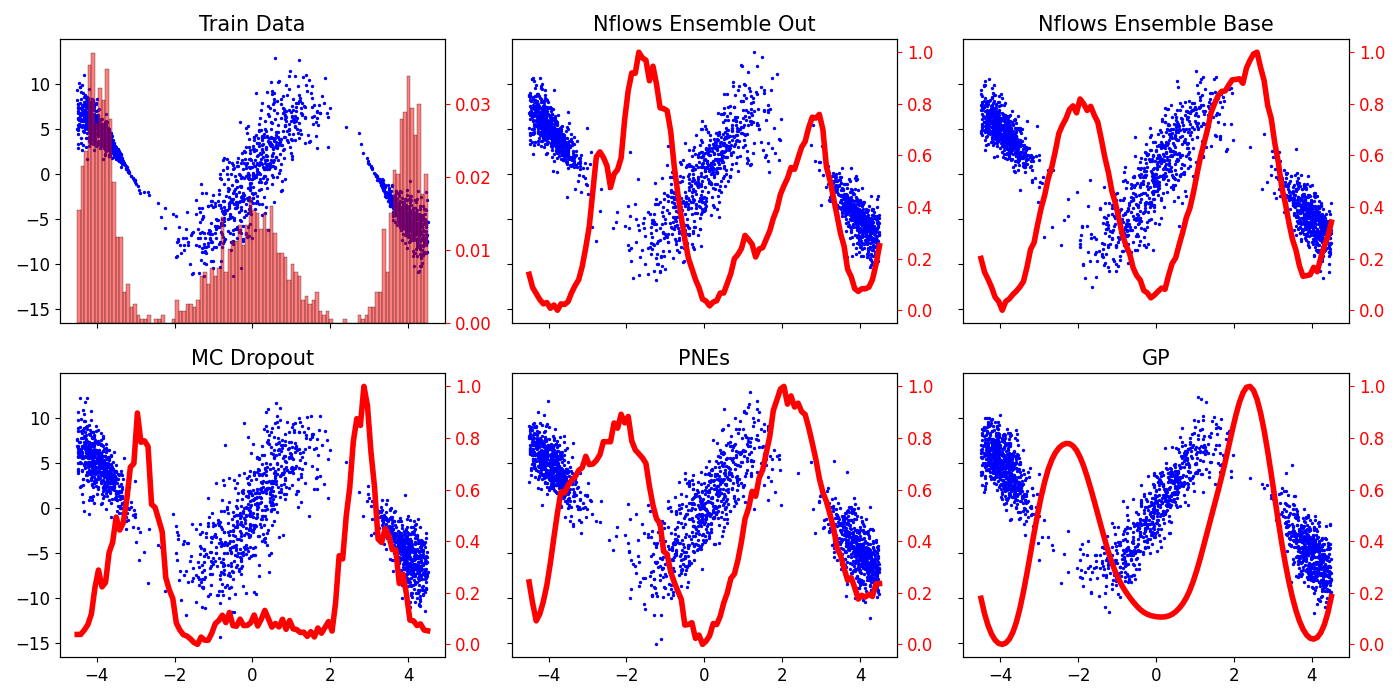} %
\caption{Samples from each model at the corresponding x-values on the \textit{Hetero} Environment are represented by the blue dots and left y-axes. Epistemic uncertainty is represented by the red curve and the right y-axes. Note that the probability density is in the top left graph instead of epistemic uncertainty}
\label{fig:hetero}
\end{figure*}

\subsection{Normalizing Flows}\label{sec:nf}
NFs are non-parametric models that have been shown to be able to fit flexible multi-modal distributions \cite{tabak2010density, tabak2013family}. They do so by transforming a simple continuous distribution (e.g. Gaussian, Beta, etc.) into a more complex one via the change of variable formula. These transformations make it so one can score and sample from the fitted distribution, thus allowing NFs to be applied to a multitude of problems. Let $B$ be a D-dimensional continuous random vector with $p_B(b)$ as its density function and let $Y=g(B)$ where $g$ is invertible, $g^{-1}$ exists, and both $g$ and $g^{-1}$ are differentiable. Using the change of variable formula, we can write the distribution of $Y$ as,
\begin{align}
    p_{Y}(y) & = p_{B}(g^{-1}(y))|\det(J(g^{-1}(y)))|, \label{eq:change_variable}
\end{align}
where $J(\cdot)$ is the Jacobian and $\det$ refers to the determinant. The first term in the product on the RHS of Equation (\ref{eq:change_variable}) is what changes the shape of the distribution while $|\det(J(g^{-1}(y)))|$ normalizes it, forcing it to integrate to one.

NF models can be learned by making $g(\cdot)$ parameterized by $\theta$, i.e., $g_{\theta}(\cdot)$, and then learned via the log likelihood. In addition, NFs can be made conditional \cite{winkler2019learning, ardizzone2019guided}. Following the framework of \citet{winkler2019learning}, Equation (\ref{eq:change_variable}) becomes,
\begin{align}
    p_{Y|X}(y|x) &= p_{B|X}(g^{-1}_{\theta}(y,x))\times\nonumber\\&|\det(J(g^{-1}_{\theta}(y,x)))|,\\
    \log(p_{Y|X}(y|x)) &= \log(p_{B|X}(g^{-1}_{\theta}(y,x))) + \nonumber \\ 
    &\log(|\det(J(g^{-1}_{\theta}(y,x)))|)\label{eq:score_nf}.
\end{align}
Note that now $g^{-1}_{\theta}:Y\times X\mapsto B$ and NF models commonly chain together multiple bijective mappings to make their models more flexible. When fitting a NF, one typically optimizes the negative log likelihood in Equation (\ref{eq:score_nf}) over mini batches. For a complete overview of NFs, please refer to \citet{papamakarios2021normalizing}.

\subsection{Ensembles}\label{sec:ensembles}
Ensembles use multiple models to obtain better predictive performance and to measure uncertainty. The conditional model can be written as,
\begin{align}
    f_{\theta}(y,x)=\sum_{w=1}^M\pi_wf_{\theta_w}(y,x),
\end{align}
where $M$ and $\pi_w$ are the number of model components and the component weights, respectively.

Ensembles are typically generated in one of two ways: randomization \cite{breiman2001random} and boosting \cite{freund1997decision}. Randomization has been preferred method for deep learning models \cite{lakshminarayanan2017simple}. Each model is randomly initialized and then at each training step for a model $w$, a sample, with replacement, is drawn from a training set $\mathcal{D}$ and then a step in the direction of the gradient is taken. This creates diversity as each ensemble component is exposed to a different portion of $\mathcal{D}$ at each step in the gradient. 
\section{Normalizing Flows Ensembles}
We propose two approaches for creating NF ensembles, Nflows Out and Nflows Base, both of which rely on neural spline bijective mappings, $g_{\theta}(\cdot)$, as they have been shown to be very flexible distribution approximators \cite{csf_durkan, nsf_durkan}. Both bagging and random initialization are utilized in training our ensemble components. Each ensemble component is created via fixed dropout masks \cite{durasov2021masksembles}, which reduces the memory and computation cost of our method.
\subsection{Nflows Out}
Nflows Out creates an ensemble in the nonlinear transformations $g$'s,
\begin{align}
    p_{Y|X,W}(y|x,w) & = f_{\theta_w}(y,x)= p_{B|X}(g_{\theta_w}^{-1}(y,x))\times \nonumber\\
    &|\det(J(g_{\theta_w}^{-1}(y,x)))|.  \label{eq:out}
\end{align} 
The base distribution is static for each component and the bijective transformation is where the component variability lies. The network $g_{\theta_w}$ outputs the parameters of cubic spline and thus each ensemble component produces a different cubic spline. \textbf{By including the complex aleatoric uncertainty prediction of an NF as well as the ability of dropout ensembles to capture uncertainty over learned parameters, our method bridges the state-of-the-art of aleatoric representation and epistemic uncertainty estimation.} These new capabilities allows Nflows Out to be applied to new decision making tasks.

Using Nflows Out, we can approximate Equation (\ref{eq:epi}). We employ Monte Carlo sampling to estimate the quantities of interest. Thus, total uncertainty is estimated as,
\begin{align}
    H(y^*|x^*)&=-E\left[\log(p_{Y|X}(y^*|x^*))\right]\label{eq:tot_uncertainty}\\
    &\approx-\frac{1}{N}\sum_{n=1}^N\log(p_{Y|X}(y_n|x^*))\label{eq:tot_uncertainty_estimate},
\end{align}
where $N$ is the number of samples drawn. For a given $x^*$ we sample $N$ points from $p_{B|X}$ and then randomly select an ensemble component $g_w$ to transform each point. The aleatoric uncertainty, $E_{p(w)}\left[H(y^*|x^*,w)\right]$, of Equation (\ref{eq:epi}) is calculated in a similar fashion,
\begin{align}
    \MoveEqLeft[3]E_{p(w)}\left[H(y^*|x^*,w)\right] \notag \\ ={}&-\frac{1}{M}\sum_{w=1}^ME\left[\log(p_{Y|X,W}(y|x^*,w))\right]\label{eq:alea}\\
    \approx{}&-\frac{1}{M}\sum_{w=1}^M\frac{1}{N_w}\Bigl(\nonumber\\
    &\sum_{n_w=1}^{N_w}\log(p_{Y|X,W}(y_{n_w}|x^*,w))\Bigr)\label{eq:alea_estimate}.
\end{align}
For a given $x^*$ and $w$, we sample $N_w$ points from $p_{B|X}$ and then transform the samples according to $g_{\theta_w}$. Monte Carlo estimation can suffer from the curse of dimensionality \cite{rubinstein2009deal}. Experiments and an intuitive proof detailing Monte Carlo's estimation limitations in high dimensions is contained in the Appendix \ref{apdx:mc_high}.

\subsection{Nflows Base}
In an attempt to alleviate the model's reliance on sampling, we propose a second ensembling technique for NFs, Nflows Base. The ensembles are created in the base distribution,
\begin{align}
    p_{Y|X,W}(y|x,w) &= f_{\theta_w}(y,x) = p_{B|X,W}(g_{\theta}^{-1}(y,x))\times\nonumber\\
    &|\det(J(g_{\theta}^{-1}(y,x)))|, \label{eq:base}
\end{align}
where $p_{B|X,W}(b|x,w)=N(\mu_w, \Sigma_w)$, $\mu_w$ and $\Sigma_w$ are estimated via a neural network with fixed dropout masks. In Nflows Base, the bijective mapping is static for each component while the base distribution varies. Figure \ref{fig:nflow_models} depicts Nflows Base on the left and Nflows Out on the right. 

\textbf{Nflows Base has advantages when estimating uncertainty as the component variability is contained in the base distribution and thus analytical formulae can be used to approximate aleatoric uncertainty.} Equation (\ref{eq:alea_estimate}) becomes,
\begin{align}
    \approx{}\frac{1}{M}\sum_{w=1}^M\frac{1}{2}\log(\det(2\pi\Sigma_w)).\label{eq:alea_estimate_base}
\end{align} 
These advantages are seen in memory reduction, computationally efficiency, and estimation error. Let $N_x$ denote the number of $x^*$ to estimate uncertainty for. Then Nflows Out needs to sample $TS$ points where $TS=N_xN_wM$ and the samples are also used to estimate Equation (\ref{eq:tot_uncertainty_estimate}) ($N_{out}=N_wM$). On the other hand, for Nflows Base, $TS=N_xN_{base}$ points are needed to estimate uncertainty, where $N_{base}=\frac{N_{out}}{M}$. This reduces the number of samples needed to drawn by a factor $M$. Please refer to the Appendix \ref{apdx:mc_high} for analysis of time savings. In addition to reducing the sampling required, there is less estimation error as component-wise entropy can be computed directly from Equation (\ref{eq:alea_estimate_base}) instead of through sampling. For nonparametric models, it is uncommon to estimate aleatoric uncertainty without sampling as the output distribution does not hold a parametric form. Note we can use the base distribution to estimate epistemic uncertainty as MI is invariant to bijective mappings \cite{kraskov2004estimating}. A proof is contained in Appendix \ref{apdx:MI_invariance}.

\begin{figure*}[t]
\centering
\includegraphics[width=0.9\textwidth]{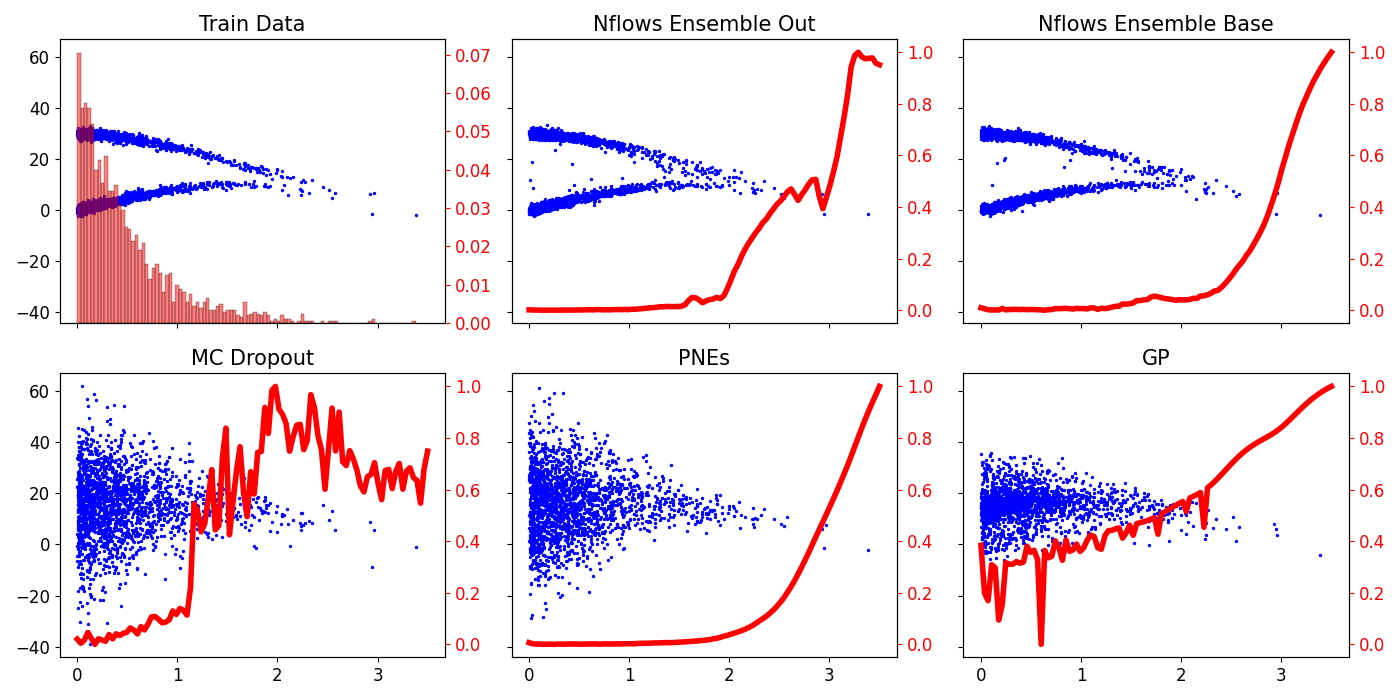} %
\caption{Samples from each model at the corresponding x-values on the \textit{Bimodal} Environment are represented by the blue dots and left y-axes. Epistemic uncertainty is represented by the red curve and the right y-axes. Note that the probability density is in the top left graph instead of epistemic uncertainty}
\label{fig:bimodal}
\end{figure*}

\section{Baseline Models}
We have included several baselines and compared each method's ability to measure aleatoric and epistemic uncertainty. These baselines are detailed below.
\subsection{Probabilistic Network Ensembles}
Probabilistic network ensembles (PNEs) have been shown to be a powerful tool to measure uncertainty for neural networks \cite{NEURIPS2018_3de568f8, kurutach2018model}. We are particularly interested in capturing their capabilities at measuring aleatoric and epistemic uncertainty in supervised learning tasks. PNEs were created with fixed dropout masks with each component modeling a Gaussian, 
\begin{align}
    p_{Y|X}(y|x) = \frac{1}{M}\sum_{w=1}^Mp_{Y|X,W}(y|x,w).\label{eq:pen}
\end{align}
The model is then trained via negative log likelihood, with randomly initialized weights and bootstrapped samples from the training set. We estimate epistemic uncertainty for PNEs via the same method for Nflows Base, Equations (\ref{eq:tot_uncertainty_estimate}) and (\ref{eq:alea_estimate_base}).

\subsection{Monte Carlo Dropout}

In addition to ensembles, other Bayesian approximations exist in the deep learning community. MC dropout is one of the more prevalent and commonly used Bayesian approximations \cite{gal2016dropout, gal2017deep, kirsch2019batchbald}.  MC dropout creates ensembles via dropout during training and uses dropout at test time to estimate uncertainty. The output distribution is therefore similar to PNEs in Equation (\ref{eq:pen}). However, we cannot sample each mask at test time, and thus a random sample of masks needs to be drawn. Therefore, when estimating uncertainty for MC dropout we first sample a set of masks and then sample each Gaussian corresponding to each mask. After which, we apply Equations (\ref{eq:tot_uncertainty_estimate}) and (\ref{eq:alea_estimate_base}) to measure uncertainty. Note the each mask has equal probability, as the dropout probability was set to 0.5.
\subsection{Gaussian Processes}
Gaussian Processes (GPs) are Bayesian models widely used to quantify uncertainty \cite{rasmussen2003gaussian}. A GP model can be fully defined by its mean function $m(\cdot)$ and a positive semidefinite covariance function/kernel $k(\cdot,\cdot)$ of a real process $f(x)$, 
\begin{align}
    m(x) &= E\left[f(x)\right], \nonumber\\
    k(x,x') &= E\left[(f(x)-m(x))(f(x')-m(x'))\right].
\end{align}
Choosing the mean and covariance function allows a practitioner to input prior knowledge into the model. From these choices, the predictive posterior becomes,
\begin{align}
    E_f[x^*] &= m(x^*)=k_*^T(K+\sigma_w^2I)^{-1}y\\
    var_f[x^*] &= k_{**}-k_*^T(K+\sigma_w^2I)k_*,
\end{align}
where $k_*=k(X,x^*)$, $k_{**}=k(x^*,x^*)$, $K$ is the kernel matrix with entries $K_{ij}=k(x_i,x_j)$, and $\sigma_w^2$ is a noise variance hyperparameter. Note that $X$ and $y$ refer to the complete set of training data. GPs place a probability distribution over functions that are possible fits over the data points. This distribution over functions is used to express uncertainty and is used to quantify epistemic uncertainty.   
\section{Experiments}
We first evaluate each method on two 1D environments, previously proposed \cite{depeweg2018decomposition}, to compare whether each method can capture multi-modal and heteroscedastic noise while measuring epistemic uncertainty. In addition, we provide 5 active learning problems, both 1D and multi-dimensional. Both openai gym and nflows libraries were utilized with minor changes \cite{brockman2016openai, nflows}. The model from \citet{depeweg2016learning,depeweg2018decomposition} is Bayesian Neural Network that is limited to mixture Gaussians and is thus not as expressive NFs for aleatoric uncertainty. In addition, there is no source code for it and was not included as a baseline, though we have included MC dropout as a close comparison. For all model hyper-parameters please refer to the Appendix \ref{apdx:model_hyperparameters} and the code can be found at https://github.com/nwaftp23/nflows\_epistemic.
\subsection{Data}
\begin{figure*}[t]
\centering
\includegraphics[width=0.9\textwidth]{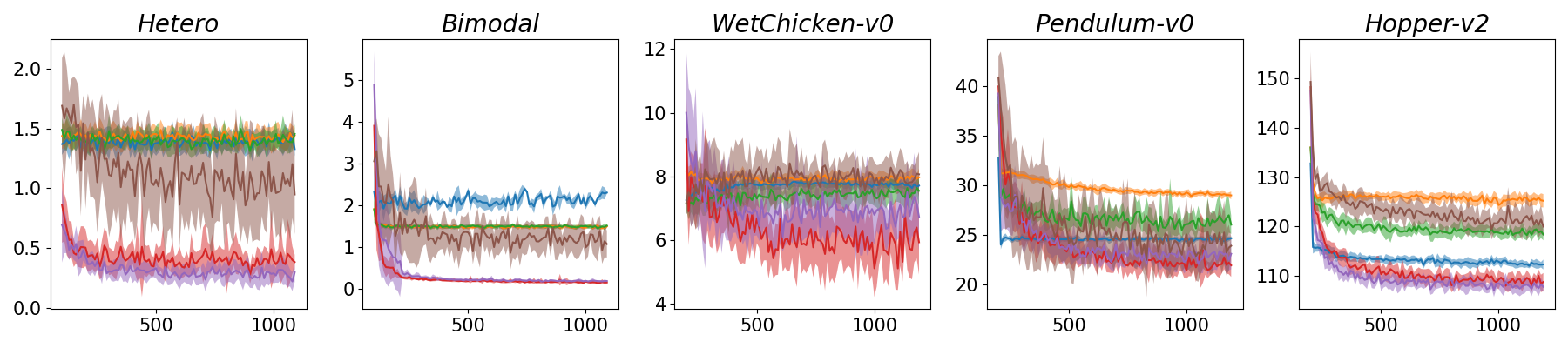} %
\begin{tabular}{c c c c c c}
\raisebox{0.7ex}{\colorbox{blue_gp}{ }} GP & \raisebox{0.7ex}{\colorbox{orange_mc_drop}{ }} MC Drop  & \raisebox{0.7ex}{\colorbox{green_pne}{ }} PNE & \raisebox{0.7ex}{\colorbox{brown_nflows}{ }} Nflows & \raisebox{0.7ex}{\colorbox{purple_nflows_out}{ }} Nflows Out & \raisebox{0.7ex}{\colorbox{red_nflows_base}{ }} Nflows Base
\end{tabular}

\caption{Mean KL divergence on 50 randomly sampled test set inputs as data was added to the training sets ($\frac{1}{50}\sum_{i=1}^{50}D_{KL}\infdivx{P_i}{Q_i}$, where $P_i$ is the ground truth conditional distribution and $Q_i$ is the model's conditional distribution conditioned on $x_i$).}
\label{fig:actlearn_allmodels}
\end{figure*}
In order to assess our uncertainty estimation, we evaluate on two 1D benchmarks, \textit{Hetero} and \textit{Bimodal} \cite{depeweg2018decomposition}. The \textit{Hetero} data can be seen in Figure \ref{fig:hetero} in the top left pane. There are two valleys of low data density displayed in the red bar chart where one would expect high epistemic uncertainty. The \textit{Bimodal} data can be seen in Figure \ref{fig:bimodal} in the top left pane. Density of data drops off when moving right along the x-axis, thus you would expect epistemic uncertainty to grow as x does. For complete details involved in generating the data, please refer to the Appendix \ref{apdx:1d_envs}. Both 1D environments were included to provide proof of concept for uncertainty estimation and visualizations.

In addition to the 1D environments, we validated our methods across three multi-dimensional environments. Trajectories were gathered from an agent and then the transition function for each environment was estimated, $f(s_{t},a_{t}) =s_{t+1}$. The first, \textit{Wet Chicken} \cite{wetchicken}, is commonly used to evaluate a model's capacity to fit multi-modal and heteroscedastic noise \cite{depeweg2018decomposition, depeweg2016learning}. It simulates a canoe approaching the edge of a waterfall. The paddlers are enticed to the edge of the waterfall as that is the region with the most fish, but as they get closer the probability increases that they fall over the edge and start over. Hence towards the edge of the waterfall the transitions become bimodal. The dynamics are naturally stochastic and are governed by the equations contained in the Appendix \ref{apdx:wet_chicken}. \textit{Wet Chicken} was included to assess uncertainty estimation on an intrinsically stochastic multi-dimensional environment.

Moreover, we evaluated all methods on \textit{Pendulum} \cite{brockman2016openai} and \textit{Hopper} \cite{todorov2012mujoco}. These environments are included because they are commonly used in benchmarking and provide us a higher dimensional output space to validate our methods. These environments are inherently nonstochactisic and thus noise was injected into the dynamics in order to produce multi-modal noise. The noise is applied to each action, $a'_{t}=a_{t}+a_{max}\epsilon$, where epsilon is drawn from a multi-modal distribution and $a_{max}$ refers to the maximum value in the action space. Note that the parameters used to create the noise distribution are included in the Appendix \ref{apdx:stochastic_mujoco} and that $a_t$ is recorded in the replay buffer, not $a'_t$.

\subsection{1D Fits}
\begin{table*}[t]
\centering
\begin{tabular}{l|l|l|l|l|l|l|l}
      Env  &   Acq Batch  & GP & PNEs & MC Drop& Nflows & Nflows Out & Nflows Base \\
      \midrule
\multirow{4}{*}{\textit{Hetero}} & 10  &    1.43±0.12 &    1.39±0.06 &     1.43±0.1 &    1.54±0.23 &    \textbf{0.48±0.09} &    \textbf{0.51±0.17} \\
          & 25  &    1.43±0.11 &     1.44±0.1 &    1.43±0.08 &     1.3±0.45 &    \textbf{0.31±0.09} &    \textbf{0.43±0.11} \\
          & 50  &    1.41±0.07 &    1.39±0.07 &     1.46±0.1 &    1.39±0.21 &    \textbf{0.27±0.08} &    \textbf{0.36±0.09} \\
          & 100 &    1.33±0.06 &    1.45±0.08 &    1.44±0.07 &    0.95±0.33 &     \textbf{0.3±0.08} &   \textbf{ 0.38±0.06} \\

       \midrule
\multirow{4}{*}{\textit{Bimodal}} & 10  &    2.23±0.18 &    1.49±0.06 &    1.49±0.03 &    1.26±0.81 &    0.74±0.76 &    \textbf{0.36±0.07} \\
          & 25  &    2.02±0.07 &    1.46±0.05 &    1.47±0.04 &     1.21±0.6 &    \textbf{0.24±0.04} &     \textbf{0.2±0.03} \\
          & 50  &    1.97±0.07 &     1.5±0.03 &    1.49±0.02 &     1.2±0.45 &    \textbf{0.22±0.03} &    \textbf{0.18±0.03} \\
          & 100 &     2.3±0.03 &    1.51±0.05 &    1.49±0.05 &    1.07±0.32 &    \textbf{0.18±0.02} &    \textbf{0.14±0.02}\\
       \midrule
\multirow{4}{*}{\textit{Wet Chicken}} & 10  &    7.61±0.21 &    \textbf{7.14±0.24} &    7.93±0.27 &    8.04±1.07 &     7.83±1.2 &    7.67±1.36 \\
          & 25  &    7.73±0.17 &    7.49±0.47 &     8.02±0.3 &    8.19±1.03 &    \textbf{6.64±0.79} &    \textbf{6.25±0.98} \\
          & 50  &    7.81±0.12 &    7.61±0.44 &    7.95±0.18 &    8.12±0.78 &    \textbf{6.51±0.56} &    \textbf{5.86±0.92} \\
          & 100 &    7.71±0.18 &    7.55±0.46 &    7.97±0.28 &     8.06±0.7 &    \textbf{6.73±1.06} &    \textbf{5.93±1.01}\\
       \midrule
\multirow{4}{*}{\textit{Pendulum-v0}} & 10  &   \textbf{24.56±0.21} &   27.29±0.73 &   31.04±0.41 &   26.45±4.61 &   27.62±1.87 &   26.07±2.22 \\
          & 25  &    \textbf{24.52±0.3} &   26.43±1.05 &   30.11±0.29 &   \textbf{24.86±3.66} &   \textbf{24.13±1.18} &   \textbf{23.97±2.16} \\
          & 50  &   24.68±0.26 &   26.47±1.19 &    29.6±0.26 &   24.44±3.04 &   \textbf{22.86±1.63} &   \textbf{22.45±1.29} \\
          & 100 &   24.67±0.17 &   26.04±0.94 &    29.0±0.39 &    23.9±0.93 &   23.09±1.56 &   \textbf{21.93±1.17} \\
          \midrule
\multirow{4}{*}{\textit{Hopper-v2}} & 10  &   114.8±0.97 &  122.42±1.22 &  125.66±0.98 &  126.87±2.83 &  \textbf{112.79±1.18} &   114.6±2.14 \\
          & 25  &  113.29±0.62 &   120.3±1.56 &  125.99±0.84 &  123.93±1.84 &   \textbf{109.31±2.0} &  \textbf{110.64±1.59} \\
          & 50  &  112.98±0.71 &  119.82±1.66 &  125.65±0.88 &  122.36±1.36 &  \textbf{108.59±1.04} &  \textbf{109.44±1.87} \\
          & 100 &   112.27±1.0 &   118.4±1.21 &  125.21±1.36 &  119.97±1.91 &  \textbf{107.74±0.99} &  \textbf{108.71±1.83} \\
       \bottomrule
\end{tabular}
\caption{KL Divergence of 50 randomly sampled test inputs ground truth distribution and the corresponding model distribution. Experiments were across ten different seeds and the results are expressed as mean plus minus one standard deviation. Best models are bolded.}
\label{tbl:kl}
\end{table*}
First, we turn to our 1D environments to give motivation and empirical proof of the epistemic uncertainty estimation. In Figure \ref{fig:hetero}, we can see that each model does a reasonable job of capturing the narrow and widening of the ground truth distribution for the \textit{Hetero} environment. Moreover, Figure \ref{fig:hetero} displays each model's capacity to estimate epistemic uncertainty in the red curve and right y-axis. Each model does a good job of capturing higher epistemic uncertainty in the two regions with fewer data points. It is not surprising to see GPs, PNEs, and MC dropout perform well on \textit{Hetero}, as each method has been shown to accurately fit heteroscedastic Gaussian noise.

\textbf{While the baseline methods were sufficiently capable of capturing the \textit{Hetero} data, this not the case for the \textit{Bimodal} setting.} Figure \ref{fig:bimodal} shows that none of the non-NF models can fit both modes. This is to be expected as GPs, PNEs and MC dropout fit a Gaussian at each $x$. While ensembles such as PNEs and MC dropout can theoretically capture multiple modes, there is no guarantee of this. To guarantee capturing multiple modes, deep learning engineering is required. Once can either manipulate of the loss function or separate the training data according to each mode, both of which may require prior knowledge. On the other hand, NFs are able to fit multiple modes directly via the log likelihood. Note that some fine tuning of the number of transformations and transformation types is required.

In addition to the aleatoric uncertainty estimation, Figure \ref{fig:bimodal} displays each model's ability to capture epistemic uncertainty on the \textit{Bimodal} data with the red curve on the right y-axis. \textbf{Nflows Out and Nflows Base are the only models to capture both modes while maintaining accurate epistemic uncertainty estimates.} Each model shows the pattern of increasing uncertainty where the data is more scarce with MC dropout having the most trouble displaying this pattern. MC dropout had the most difficulty to estimate epistemic uncertainty on most tasks. This is likely the case because, as opposed to the PNEs or the NFs methods, the number of ensemble components is generally quite larger ($2^n$ where $n$ is the number of neurons in your network). Therefore, when estimating Equation (\ref{eq:epi}), a subset of masks needs to be sampled, leading to less stable estimates of uncertainty.

\subsection{Active Learning}

In order to assess each model's capability at utilizing epistemic and capturing flexible aleatoric uncertainty, we provide an active learning experimental setup. For both 1D and multi-dimensional experiments, each model started with 100 and 200 data points, respectively. At each epoch, the models sampled 1000 new inputs and kept the 10 with the highest epistemic uncertainty. The models are then evaluated by sampling 50 inputs from the test set and averaging the Kullback-Liebler (KL) divergence for the ground truth distribution at those points and model's distributions,
\begin{align}
    \frac{1}{50}\sum_{i=1}^{50}D_{KL}\infdivx{P_i}{Q_i}
\end{align}
Note that the KL divergences reported were estimated via samples using the $k$-nearest-neighbor method \cite{knnkl}. KL divergence was chosen as an evaluation metric as we are most interested in distributional fit. In order to ensure variation, the training and test were gathered via different processes; refer to the Appendix \ref{apdx:train_vs_test} for more details. All experiments were run across 10 seeds and their mean and standard deviation are reported.

Figure \ref{fig:actlearn_allmodels} displays the performance of each method as the training size increases. \textbf{For each data setting, the NF ensemble models reach lower KLs, thus verifying that they can leverage epistemic uncertainty estimation to learn more expressive aleatoric uncertainty faster.} In some cases other models provided better results with small number of data points, this information is conveyed in Table \ref{tbl:kl}, with the best-performing models in bold at different acquisition epochs. Note that in addition to the baselines discussed, we included an NF with no ensembles, using total entropy as the acquisition function. 

\section{Related Work}

Using Bayesian methods, researchers have developed information based criterion for the problem of active learning using Gaussian Processes (GPs) on classification problems \cite{houlsby2011bayesian}. Researchers have leveraged said information based criterion for uncertainty estimation with Bayesian neural networks \cite{gal2017deep, kendall2017uncertainties, kirsch2019batchbald}. These works extended previous epistemic uncertainty estimation by leveraging dropout to estimate uncertainty on image classification for neural networks. In contrast, our work estimates epistemic uncertainty on a harder output space. The experiments contained in this paper were conducted on regression problems where the output is drawn from continuous distributions in 1-11 dimensions, whereas the previous works applied their methods to classification problems, a 1D categorical output.

NFs have been shown to be poor estimators of epistemic uncertainty \cite{kirichenko2020normalizing, zhang2021understanding}. Researchers have argued that NFs, inability to differentiate between in and out distribution samples via their likelihood is a methodological shortcoming. Some have found workarounds to this problem, specifically in the sphere of ensembles \cite{choi2018waic}. Ensembles have been shown to be a powerful tool in the belt of a machine learning practitioner by leveraging similar uncertainty quantification benefits to their Bayesian cousins but at a smaller computational footprint \cite{lakshminarayanan2017simple}. The work regarding NFs and uncertainty have focused on image generation and unsupervised learning. Our methods differ, as we consider supervised learning problems \cite{winkler2019learning}. In addition, the ensembles created in \citet{choi2018waic} contrast with ours as we leverage the base distribution to estimate our uncertainty, use MI instead of WAIC and create our ensembles with less memory.

In addition to the examples discussed, work has been done to quantify epistemic uncertainty for regression problems \cite{depeweg2018decomposition, postels2020hidden}. \citet{depeweg2016learning} method's relied on Bayesian approximations to neural networks which modeled mixture of Gaussians and demonstrated their ability to capture uncertainty on three environments. Our work expands on this, by developing more expressive NF models for uncertainty estimation. \citet{postels2020hidden} develop theory to show how latent representations of a proxy network can be used to estimate uncertainty. They use the proxy network's latent representation as their conditioner. In contrast, we show to how to do this with one NF model and how to leverage the base distribution to be more sample efficient. This study provided an array of multi-modal problems while \citet{postels2020hidden} considered a single uni-modal problem. Our work expands on both these papers by providing a comprehensive analysis of different baseline methods, comparing their uncertainty quantification and including higher dimensional data. In addition, we provide a full active learning experimental setup and develop new NF frameworks for measuring uncertainty that are more sample efficient and have lower memory costs.

\section{Conclusion}
In this paper, we introduced NF ensembles via fixed dropout masks and demonstrated how they can be used efficiently to quantify uncertainty. In doing so, we show how to leverage the base distribution to estimate uncertainty more sample efficiently. Moreover, Nflows Base shows that one can accurately measure uncertainty in the base distribution space. We empirically show that our models outperform the state-of-the-art in capturing the combination of aleatoric and epistemic uncertainty on 5 regression tasks. This paper shows that NF ensembles are an expressive model for aleatoric uncertainty while keeping the benefits of previous methods for capturing epistemic uncertainty. 

\section*{Acknowledgements}
Prof. Meger was supported by the National Sciences and Engineering Research Council (NSERC), through the NSERC Canadian Robotics Network (NCRN).
\bibliography{refs.bib}

\begin{thebibliography}{39}
\providecommand{\natexlab}[1]{#1}

\bibitem[{Ardizzone et~al.(2019)Ardizzone, L{\"u}th, Kruse, Rother, and
  K{\"o}the}]{ardizzone2019guided}
Ardizzone, L.; L{\"u}th, C.; Kruse, J.; Rother, C.; and K{\"o}the, U. 2019.
\newblock Guided image generation with conditional invertible neural networks.
\newblock \emph{arXiv preprint arXiv:1907.02392}.

\bibitem[{Breiman(2001)}]{breiman2001random}
Breiman, L. 2001.
\newblock Random forests.
\newblock \emph{Machine learning}, 45(1): 5--32.

\bibitem[{Brockman et~al.(2016)Brockman, Cheung, Pettersson, Schneider,
  Schulman, Tang, and Zaremba}]{brockman2016openai}
Brockman, G.; Cheung, V.; Pettersson, L.; Schneider, J.; Schulman, J.; Tang,
  J.; and Zaremba, W. 2016.
\newblock Openai gym.
\newblock \emph{arXiv preprint arXiv:1606.01540}.

\bibitem[{Choi, Jang, and Alemi(2018)}]{choi2018waic}
Choi, H.; Jang, E.; and Alemi, A.~A. 2018.
\newblock Waic, but why? generative ensembles for robust anomaly detection.
\newblock \emph{arXiv preprint arXiv:1810.01392}.

\bibitem[{Chua et~al.(2018)Chua, Calandra, McAllister, and
  Levine}]{NEURIPS2018_3de568f8}
Chua, K.; Calandra, R.; McAllister, R.; and Levine, S. 2018.
\newblock Deep Reinforcement Learning in a Handful of Trials using
  Probabilistic Dynamics Models.
\newblock In \emph{Advances in Neural Information Processing Systems},
  volume~31.

\bibitem[{Depeweg et~al.(2016)Depeweg, Hern{\'a}ndez-Lobato, Doshi-Velez, and
  Udluft}]{depeweg2016learning}
Depeweg, S.; Hern{\'a}ndez-Lobato, J.~M.; Doshi-Velez, F.; and Udluft, S. 2016.
\newblock Learning and policy search in stochastic dynamical systems with
  bayesian neural networks.
\newblock \emph{arXiv preprint arXiv:1605.07127}.

\bibitem[{Depeweg et~al.(2018)Depeweg, Hernandez-Lobato, Doshi-Velez, and
  Udluft}]{depeweg2018decomposition}
Depeweg, S.; Hernandez-Lobato, J.-M.; Doshi-Velez, F.; and Udluft, S. 2018.
\newblock Decomposition of uncertainty in Bayesian deep learning for efficient
  and risk-sensitive learning.
\newblock In \emph{International Conference on Machine Learning}, 1184--1193.
  PMLR.

\bibitem[{Der~Kiureghian and Ditlevsen(2009)}]{der2009aleatory}
Der~Kiureghian, A.; and Ditlevsen, O. 2009.
\newblock Aleatory or epistemic? Does it matter?
\newblock \emph{Structural safety}, 31(2): 105--112.

\bibitem[{Durasov et~al.(2021)Durasov, Bagautdinov, Baque, and
  Fua}]{durasov2021masksembles}
Durasov, N.; Bagautdinov, T.; Baque, P.; and Fua, P. 2021.
\newblock Masksembles for uncertainty estimation.
\newblock In \emph{Proceedings of the IEEE/CVF Conference on Computer Vision
  and Pattern Recognition}, 13539--13548.

\bibitem[{Durkan et~al.(2019{\natexlab{a}})Durkan, Bekasov, Murray, and
  Papamakarios}]{csf_durkan}
Durkan, C.; Bekasov, A.; Murray, I.; and Papamakarios, G. 2019{\natexlab{a}}.
\newblock Cubic-Spline Flows.
\newblock In \emph{Workshop on Invertible Neural Networks and Normalizing
  Flows, International Conference on Machine Learning}.

\bibitem[{Durkan et~al.(2019{\natexlab{b}})Durkan, Bekasov, Murray, and
  Papamakarios}]{nsf_durkan}
Durkan, C.; Bekasov, A.; Murray, I.; and Papamakarios, G. 2019{\natexlab{b}}.
\newblock Neural Spline Flows.
\newblock In \emph{Advances in Neural Information Processing Systems},
  volume~32.

\bibitem[{Durkan et~al.(2020)Durkan, Bekasov, Murray, and
  Papamakarios}]{nflows}
Durkan, C.; Bekasov, A.; Murray, I.; and Papamakarios, G. 2020.
\newblock {nflows}: normalizing flows in {PyTorch}.
\newblock \url{https://doi.org/10.5281/zenodo.4296287}.
\newblock Accessed: 2021-09-01.

\bibitem[{Freund and Schapire(1997)}]{freund1997decision}
Freund, Y.; and Schapire, R.~E. 1997.
\newblock A decision-theoretic generalization of on-line learning and an
  application to boosting.
\newblock \emph{Journal of computer and system sciences}, 55(1): 119--139.

\bibitem[{Gal and Ghahramani(2016)}]{gal2016dropout}
Gal, Y.; and Ghahramani, Z. 2016.
\newblock Dropout as a bayesian approximation: Representing model uncertainty
  in deep learning.
\newblock In \emph{international conference on machine learning}, 1050--1059.
  PMLR.

\bibitem[{Gal, Islam, and Ghahramani(2017)}]{gal2017deep}
Gal, Y.; Islam, R.; and Ghahramani, Z. 2017.
\newblock Deep bayesian active learning with image data.
\newblock In \emph{International Conference on Machine Learning}, 1183--1192.
  PMLR.

\bibitem[{Hora(1996)}]{hora1996aleatory}
Hora, S.~C. 1996.
\newblock Aleatory and epistemic uncertainty in probability elicitation with an
  example from hazardous waste management.
\newblock \emph{Reliability Engineering \& System Safety}, 54(2-3): 217--223.

\bibitem[{Houlsby et~al.(2011)Houlsby, Husz{\'a}r, Ghahramani, and
  Lengyel}]{houlsby2011bayesian}
Houlsby, N.; Husz{\'a}r, F.; Ghahramani, Z.; and Lengyel, M. 2011.
\newblock Bayesian active learning for classification and preference learning.
\newblock \emph{arXiv preprint arXiv:1112.5745}.

\bibitem[{H{\"u}llermeier and Waegeman(2021)}]{hullermeier2021aleatoric}
H{\"u}llermeier, E.; and Waegeman, W. 2021.
\newblock Aleatoric and epistemic uncertainty in machine learning: An
  introduction to concepts and methods.
\newblock \emph{Machine Learning}, 110(3): 457--506.

\bibitem[{Kendall and Gal(2017)}]{kendall2017uncertainties}
Kendall, A.; and Gal, Y. 2017.
\newblock What uncertainties do we need in bayesian deep learning for computer
  vision?
\newblock \emph{Advances in neural information processing systems}, 30.

\bibitem[{Kingma and Dhariwal(2018)}]{kingma2018glow}
Kingma, D.~P.; and Dhariwal, P. 2018.
\newblock Glow: Generative flow with invertible 1x1 convolutions.
\newblock \emph{Advances in neural information processing systems}, 31.

\bibitem[{Kirichenko, Izmailov, and Wilson(2020)}]{kirichenko2020normalizing}
Kirichenko, P.; Izmailov, P.; and Wilson, A.~G. 2020.
\newblock Why normalizing flows fail to detect out-of-distribution data.
\newblock \emph{Advances in neural information processing systems}, 33:
  20578--20589.

\bibitem[{Kirsch, Van~Amersfoort, and Gal(2019)}]{kirsch2019batchbald}
Kirsch, A.; Van~Amersfoort, J.; and Gal, Y. 2019.
\newblock Batchbald: Efficient and diverse batch acquisition for deep bayesian
  active learning.
\newblock \emph{Advances in neural information processing systems}, 32.

\bibitem[{Kraskov, St{\"o}gbauer, and
  Grassberger(2004)}]{kraskov2004estimating}
Kraskov, A.; St{\"o}gbauer, H.; and Grassberger, P. 2004.
\newblock Estimating mutual information.
\newblock \emph{Physical review E}, 69(6): 066138.

\bibitem[{Kurutach et~al.(2018)Kurutach, Clavera, Duan, Tamar, and
  Abbeel}]{kurutach2018model}
Kurutach, T.; Clavera, I.; Duan, Y.; Tamar, A.; and Abbeel, P. 2018.
\newblock Model-ensemble trust-region policy optimization.
\newblock \emph{arXiv preprint arXiv:1802.10592}.

\bibitem[{Lakshminarayanan, Pritzel, and
  Blundell(2017)}]{lakshminarayanan2017simple}
Lakshminarayanan, B.; Pritzel, A.; and Blundell, C. 2017.
\newblock Simple and scalable predictive uncertainty estimation using deep
  ensembles.
\newblock \emph{Advances in neural information processing systems}, 30.

\bibitem[{MacKay(1992)}]{mackay1992information}
MacKay, D.~J. 1992.
\newblock Information-based objective functions for active data selection.
\newblock \emph{Neural computation}, 4(4): 590--604.

\bibitem[{Papamakarios et~al.(2021)Papamakarios, Nalisnick, Rezende, Mohamed,
  and Lakshminarayanan}]{papamakarios2021normalizing}
Papamakarios, G.; Nalisnick, E.~T.; Rezende, D.~J.; Mohamed, S.; and
  Lakshminarayanan, B. 2021.
\newblock Normalizing Flows for Probabilistic Modeling and Inference.
\newblock \emph{J. Mach. Learn. Res.}, 22(57): 1--64.

\bibitem[{Postels et~al.(2020)Postels, Blum, Str{\"u}mpler, Cadena, Siegwart,
  Van~Gool, and Tombari}]{postels2020hidden}
Postels, J.; Blum, H.; Str{\"u}mpler, Y.; Cadena, C.; Siegwart, R.; Van~Gool,
  L.; and Tombari, F. 2020.
\newblock The hidden uncertainty in a neural networks activations.
\newblock \emph{arXiv preprint arXiv:2012.03082}.

\bibitem[{Rasmussen(2003)}]{rasmussen2003gaussian}
Rasmussen, C.~E. 2003.
\newblock Gaussian processes in machine learning.
\newblock In \emph{Summer school on machine learning}, 63--71. Springer.

\bibitem[{Rezende and Mohamed(2015)}]{nf-rezende15}
Rezende, D.; and Mohamed, S. 2015.
\newblock Variational Inference with Normalizing Flows.
\newblock In \emph{Proceedings of the 32nd International Conference on Machine
  Learning}, volume~37 of \emph{Proceedings of Machine Learning Research},
  1530--1538. Lille, France.

\bibitem[{Rubinstein and Glynn(2009)}]{rubinstein2009deal}
Rubinstein, R.~Y.; and Glynn, P.~W. 2009.
\newblock How to deal with the curse of dimensionality of likelihood ratios in
  Monte Carlo simulation.
\newblock \emph{Stochastic Models}, 25(4): 547--568.

\bibitem[{Settles(2009)}]{settles2009active}
Settles, B. 2009.
\newblock Active learning literature survey.
\newblock \emph{Synthesis Lectures on Artificial Intelligence and Machine
  Learning}.

\bibitem[{Tabak and Turner(2013)}]{tabak2013family}
Tabak, E.~G.; and Turner, C.~V. 2013.
\newblock A family of nonparametric density estimation algorithms.
\newblock \emph{Communications on Pure and Applied Mathematics}, 66(2):
  145--164.

\bibitem[{Tabak and Vanden-Eijnden(2010)}]{tabak2010density}
Tabak, E.~G.; and Vanden-Eijnden, E. 2010.
\newblock Density estimation by dual ascent of the log-likelihood.
\newblock \emph{Communications in Mathematical Sciences}, 8(1): 217--233.

\bibitem[{Todorov, Erez, and Tassa(2012)}]{todorov2012mujoco}
Todorov, E.; Erez, T.; and Tassa, Y. 2012.
\newblock Mujoco: A physics engine for model-based control.
\newblock In \emph{2012 IEEE/RSJ international conference on intelligent robots
  and systems}, 5026--5033. IEEE.

\bibitem[{Trep.(1994)}]{wetchicken}
Trep., V. 1994.
\newblock The wet game of chicken.
\newblock Technical report, Seimens AG.

\bibitem[{Wang, Kulkarni, and Verdu(2009)}]{knnkl}
Wang, Q.; Kulkarni, S.~R.; and Verdu, S. 2009.
\newblock Divergence Estimation for Multidimensional Densities Via
  $k$-Nearest-Neighbor Distances.
\newblock \emph{IEEE Transactions on Information Theory}, 55(5): 2392--2405.

\bibitem[{Winkler et~al.(2019)Winkler, Worrall, Hoogeboom, and
  Welling}]{winkler2019learning}
Winkler, C.; Worrall, D.; Hoogeboom, E.; and Welling, M. 2019.
\newblock Learning likelihoods with conditional normalizing flows.
\newblock \emph{arXiv preprint arXiv:1912.00042}.

\bibitem[{Zhang, Goldstein, and Ranganath(2021)}]{zhang2021understanding}
Zhang, L.; Goldstein, M.; and Ranganath, R. 2021.
\newblock Understanding failures in out-of-distribution detection with deep
  generative models.
\newblock In \emph{International Conference on Machine Learning}, 12427--12436.
  PMLR.

\end{thebibliography}
\clearpage
\appendix
\section{Appendix}
Please find the supplementary math and figures contained below.

\subsection{Model Hyper-Parameters}
\label{apdx:model_hyperparameters}
The PNEs used an architecture of three hidden layers, each with 50 units and ReLU activation functions. MC dropout models had five hidden layers, each with 400 hidden units and ReLU activation functions. This remained constant across all experimental settings. The hyper-parameters for Nflows Out, Nflows Base and Nflows varied across experiments and therefore are described in Table \ref{tbl:hyper-params}. Note that the Nflows hyper-parameters were chosen such that on average they had the same number of parameters as the ensemble components, as the dropout weight was 0.5. The GP used a Radial Basis Function kernel for its covariance function and constant mean for its mean function. We trained all models with 16GB ram on Intel Gold 6148 Skylake @ 2.4 GHz CPUs and NVidia V100SXM2 (16G memory) GPUs. This was accomplished thanks to the Digital Research Alliance of Canada.

PNEs, Nflows Out and Nflows Base were run with five ensemble components for each experimental setting. To estimate uncertainty at each x conditioned on Nflows Base sampled 1000, Nflows Out 5000, PNEs 5000 and MC drop 2500 points. MC Dropout sampled 20 dropout masks upon which to estimate uncertainty. We found that sampling more became computationally prohibitive and did not increase accuracy of the uncertainty estimates. 
\begin{table}[ht!]
\centering
\begin{tabular}{l|l|l|l|l}
      Model  & Env & Hids & Units & $g$'s\\
      \midrule
\multirow{4}{*}{Nflows} &  \textit{Hetero} &  1  & 10 & 1 \\
          & \textit{Bimodal} &  1 &  100 & 1 \\
          & \textit{WetChicken} &  1 &  100 & 1 \\
          & \textit{Pendulum-v0} &  2 &  10 & 2 \\
          & \textit{Hopper-v2} &  2 &  10 & 2 \\
      \midrule
\multirow{4}{*}{Nflows Out} &  \textit{Hetero} &  1  & 20 & 1 \\
          & \textit{Bimodal} &  1 &  200 & 1 \\
          & \textit{WetChicken} &  1 &  200 & 1 \\
          & \textit{Pendulum-v0} &  2 &  20 & 2 \\
          & \textit{Hopper-v2} &  2 &  20 & 2 \\

      \midrule
\multirow{4}{*}{Nflows Base} &  \textit{Hetero} &  1  & 20 & 1 \\
          & \textit{Bimodal} &  1 &  200 & 1 \\
          & \textit{WetChicken} &  1 &  200 & 1 \\
          & \textit{Pendulum-v0} &  2 &  20 & 2 \\
          & \textit{Hopper-v2} &  2 &  20 & 2 \\
      \midrule
\multirow{4}{*}{MC Drop} &  \textit{Hetero} &  5  & 400 & NA \\
          & \textit{Bimodal} &  5 &  400 & NA \\
          & \textit{WetChicken} &  5 &  400 & NA \\
          & \textit{Pendulum-v0} &  5 &  400 & NA \\
          & \textit{Hopper-v2} &  5 &  400 & NA \\
      \midrule
\multirow{4}{*}{PNEs} &  \textit{Hetero} &  3  & 50 & NA \\
          & \textit{Bimodal} &  3 &  50 & NA \\
          & \textit{WetChicken} &  3 &  50 & NA \\
          & \textit{Pendulum-v0} &  3 &  50 & NA \\
          & \textit{Hopper-v2} &  3 &  50 & NA \\
       \bottomrule
\end{tabular}
\caption{Hyper-Parameters for models.}
\label{tbl:hyper-params}
\end{table}

\subsection{1D Environments}
\label{apdx:1d_envs}
The \textit{Hetero} data was generated via sampling $n$ points from a categorical distribution with 3 values where $p_i=\frac{1}{3}$ and then drawing $x$ from one of three different Gaussians ($N(-4,\frac{2}{5})$, $N(0,\frac{9}{10})$, $N(4,\frac{2}{5})$) depending on the value of $n$. Then $y$ was generated as,
\begin{align}
    y = 7\sin(x)+3z\left|\cos\left(\frac{x}{2}\right)\right| \label{eq:hetero_data}.
\end{align}
The \textit{Bimodal} data was generated by sampling $x$ from an exponential with $\lambda=2$, then $n$ was sampled from a Bernoulli with $p=0.5$ and according to the value of $n$,
\begin{align}
    y = \begin{cases} 
      10\sin(x)+z & n =0 \\
      10\cos(x)+z+20-x & n=1 
   \end{cases}.
\end{align}
Note that for both \textit{Bimodal} and \textit{Hetero} data $z \sim N(0,1)$.
\subsection{Wet Chicken}
\label{apdx:wet_chicken}
The environment can be viewed as a group of paddlers, the agent, are paddling a canoe on a 2D river. The paddlers position, the state, at time $t$ is represented by $(x_t, y_t)$. The river has a length $l$ and width $w$, both of which were set to 5. At the end of the river, in the y-direction, there is a waterfall which forces you to return to the origin $(0,0)$ when going over. Higher rewards can be found at the edge of the river, $r_t=-(l-y_t)$, and thus the paddlers are incentivized to get as close to edge as possible. The agent can take actions $(a_{t,x}, a_{t,y}) \in [-1,1]^2$ which represents the magnitudes of paddling in each direction. In addition to the action, there are turbulences $s_t$ and drifts $v_t$ which behave stochastically.

The dynamics are governed by the following system of equations,
\begin{align*}
    v_t &= \frac{3}{w}x_t\\
    s_t &= 3.5 - v_t\\
    x_{t+1} &=\begin{cases} 
      0 & x_t+a_{t,x}<0 \\
      0 & \hat{y}_{t+1}>l\\
      w & x_t+a_{t,x}>w\\
      x_t+a_{t,x} & o/w
   \end{cases}\\
   y_{t+1} &=\begin{cases} 
      0 & y_t+a_{t,y}<0 \\
      0 & \hat{y}_{t+1}>l\\
      \hat{y}_{t+1} & o/w
   \end{cases}
\end{align*}
where $\hat{y}_{t+1}=y_t+(a_{t,y}-1)+v_t+s_t\tau_t$ and $\tau_t\sim \text{Unif}(-1,1)$. As the paddlers approach the edge of the waterfall, there transition becomes increasingly bimodal as they might fall over and return to the start. Furthermore the transitions are heteroscedastic, as $x_t$ decreases the effect of $\tau_t$ increases.
\begin{figure}[ht]
\centering
\includegraphics[width=0.9\columnwidth]{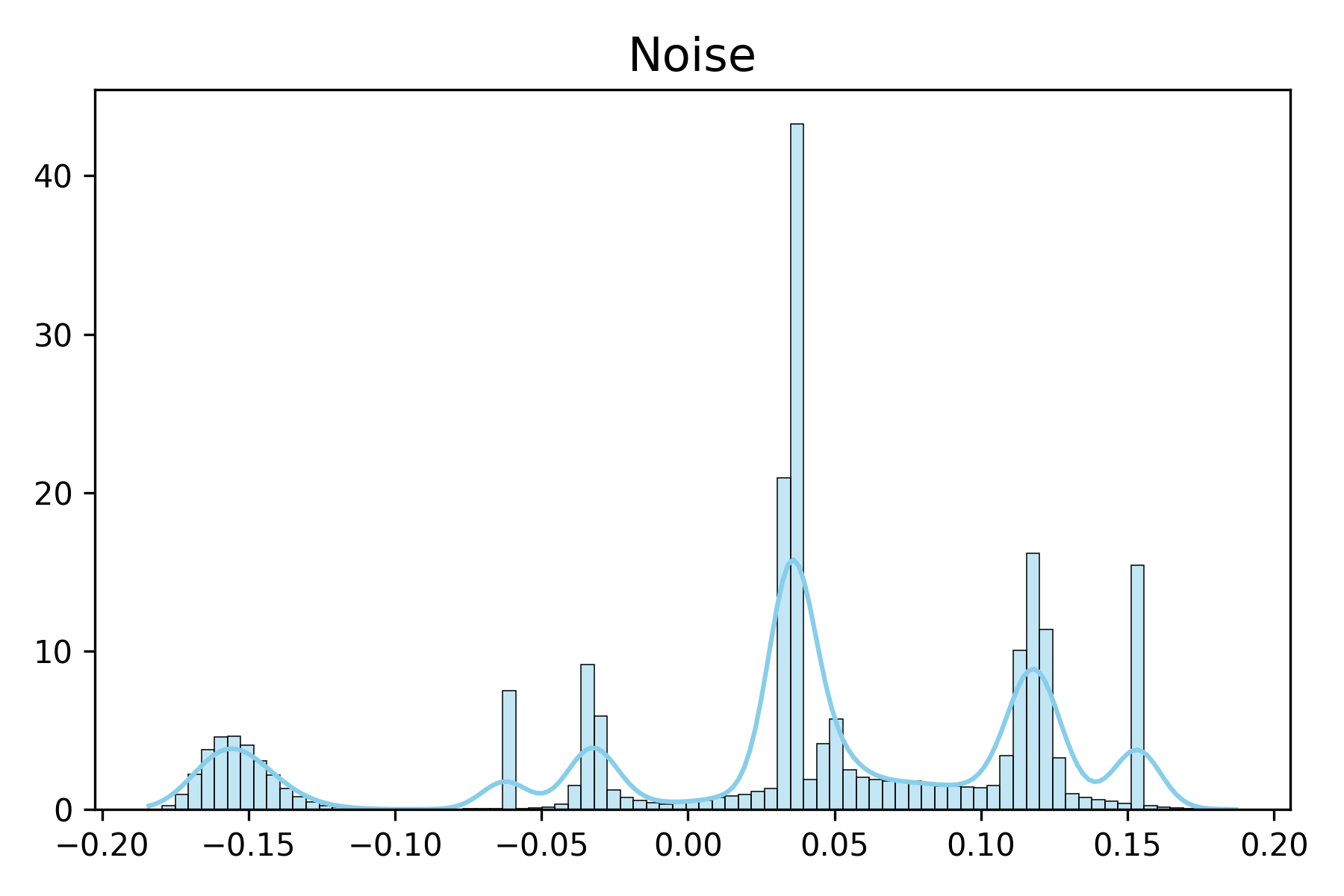} %
\caption{Sample from noise distribution that was used to introduce stochasticity into Hopper and Pendulum.}
\label{fig:noise_openai}
\end{figure}
\begin{table*}[t]
\centering
\begin{tabular}{l|l|l|l|l|l|l|l}
      Env  &   Acq Batch  & GP & PNEs & MC\_drop& Nflows & Nflows\_out & Nflows\_base \\
      \midrule
\multirow{4}{*}{\textit{Hetero}} & 10  &   5.92±0.1 &   5.89±0.17 &   5.69±0.17 &  5.57±0.28 &  \textbf{2.01±0.12} &    \textbf{2.17±0.1} \\
          & 25  &  5.95±0.06 &   5.91±0.13 &   5.66±0.14 &  4.74±1.57 &  \textbf{1.93±0.05} &   \textbf{2.22±0.32} \\
          & 50  &  5.96±0.04 &   5.88±0.06 &   5.63±0.14 &   5.3±0.63 &  \textbf{1.92±0.08} &   \textbf{2.15±0.16} \\
          & 100 &  5.96±0.04 &   5.91±0.09 &   5.67±0.13 &  4.18±1.37 &  \textbf{1.93±0.07} &   \textbf{2.19±0.18}\\

       \midrule
\multirow{4}{*}{\textit{Bimodal}} & 10  &  9.19±0.24 &  10.22±0.24 &   10.3±0.27 &  7.69±0.87 &   \textbf{6.63±0.1} &   \textbf{6.61±0.06} \\
          & 25  &   9.3±0.12 &  10.12±0.19 &  10.22±0.19 &  9.12±2.13 &  \textbf{6.62±0.08} &   \textbf{6.67±0.06} \\
          & 50  &  9.23±0.08 &  10.11±0.14 &  10.19±0.16 &  9.62±2.97 &  \textbf{6.63±0.11} &   \textbf{6.64±0.06} \\
          & 100 &  9.17±0.07 &  10.11±0.13 &  10.15±0.12 &  9.01±1.47 &  \textbf{6.63±0.04} &    \textbf{6.72±0.1}\\
       \midrule
\multirow{4}{*}{\textit{WetChicken}} & 10  &  1.35±0.06 &   1.33±0.09 &   1.38±0.08 &  1.32±0.18 &  \textbf{0.85±0.07} &   \textbf{0.89±0.09} \\
          & 25  &  1.38±0.03 &   1.36±0.07 &   1.39±0.04 &   1.37±0.2 &  \textbf{0.88±0.12} &   \textbf{0.96±0.17} \\
          & 50  &  1.39±0.03 &   1.38±0.09 &   1.42±0.04 &  1.41±0.19 &  \textbf{0.88±0.09} &   \textbf{0.89±0.08} \\
          & 100 &  1.38±0.02 &    1.4±0.08 &   1.45±0.07 &  1.42±0.21 &   \textbf{0.9±0.05} &   \textbf{0.94±0.07} \\
       \midrule
\multirow{4}{*}{\textit{Pendulum-v0}} & 10  &   0.57±0.0 &   0.68±0.08 &    1.0±0.06 &  0.46±0.38 &   \textbf{0.15±0.1}&   \textbf{0.17±0.14} \\
          & 25  &   0.57±0.0 &   0.61±0.02 &   0.84±0.05 &  0.28±0.39 &  \textbf{0.08±0.04} &   \textbf{0.09±0.05} \\
          & 50  &   0.56±0.0 &    0.6±0.04 &   0.75±0.02 &  0.38±0.44 &  \textbf{0.05±0.03}&   \textbf{0.06±0.03} \\
          & 100 &   0.56±0.0 &    0.6±0.03 &   0.71±0.01 &   0.2±0.14 &  \textbf{0.05±0.02} &   \textbf{0.06±0.04}\\
          \midrule
\multirow{4}{*}{\textit{Hopper-v2}} & 10  &  1.81±0.01 &   1.87±0.05 &   1.86±0.05 &  1.44±0.76 &  \textbf{0.48±0.09} &    \textbf{0.59±0.1} \\
          & 25  &   1.81±0.0 &   1.86±0.05 &    1.91±0.1 &  1.06±0.21 &  \textbf{0.39±0.06} &   \textbf{0.42±0.09} \\
          & 50  &   1.81±0.0 &   1.85±0.03 &   1.94±0.06 &  1.05±0.19 &  \textbf{0.33±0.03} &   \textbf{0.36±0.04} \\
          & 100 &   1.81±0.0 &   1.83±0.01 &   1.98±0.08 &  0.96±0.18 &  \textbf{0.29±0.02} &   \textbf{0.31±0.03} \\
       \bottomrule
\end{tabular}
\caption{A sample of $\hat{y}$'s is drawn from the model, given $x$, and then the RMSE between the sample and the groundtruth is calculated. This was done for 50 pairs $(x,y)$ in the test set. Experiments were across ten different seeds and the results are expressed as mean plus minus one standard deviation.}
\label{tbl:rmse}
\end{table*}
\begin{figure*}[ht!]
\centering
\includegraphics[width=0.9\textwidth]{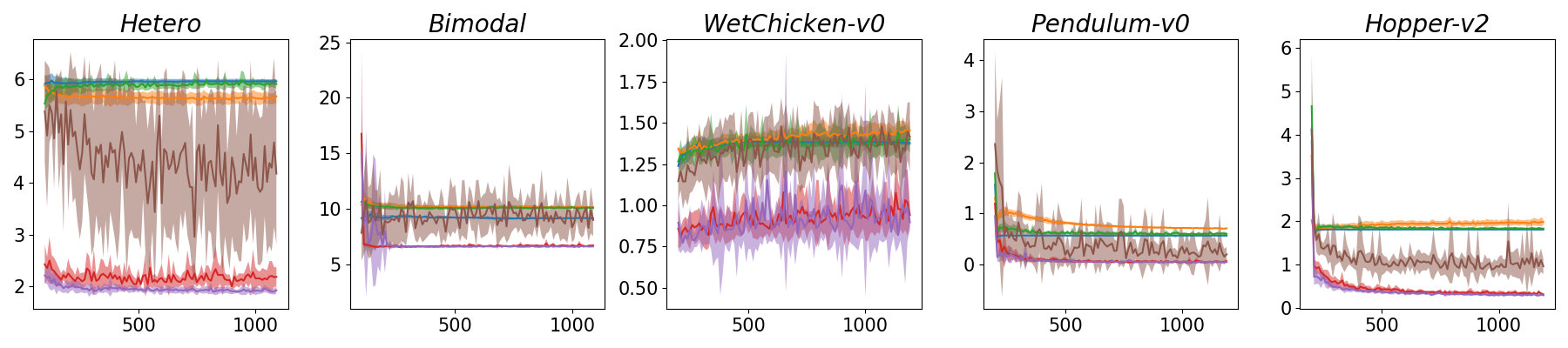} %
\begin{tabular}{c c c c c c}
\raisebox{0.7ex}{\colorbox{blue_gp}{ }} GP & \raisebox{0.7ex}{\colorbox{orange_mc_drop}{ }} MC Drop  & \raisebox{0.7ex}{\colorbox{green_pne}{ }} PNE & \raisebox{0.7ex}{\colorbox{brown_nflows}{ }} Nflows & \raisebox{0.7ex}{\colorbox{purple_nflows_out}{ }} Nflows Out & \raisebox{0.7ex}{\colorbox{red_nflows_base}{ }} Nflows Base
\end{tabular}

\caption{Mean RMSE on 50 randomly sampled test set inputs as data was added to the training sets.}
\label{fig:actlearn_allmodels_rmse}
\end{figure*}
\subsection{Hopper \& Pendulum}
\label{apdx:stochastic_mujoco}
Stochasticity was introduced into both environments via a mixture of Gaussians. The mixture weights were sampled from a Dirichlet distribution where $\alpha$ was set to a vector of ones, the means of the Gaussians were randomly drawn from a $N(0,1)$ and the standard deviations of the Gaussians were randomly drawn from a $N(0,0.5)$ and then squared.
\begin{align*}
    \pi&=[0.062, 0.128 , 0.177, 0.001, 0.032, 0.273, \\&0.062, 0.033, 0.067, 0.022, 0.142]\\
    \mu&=[0.508, -2.059,  1.355 , -0.675,  0.504, \\&0.358, -0.332, -0.647, 2.029, -0.294, 0.868]\\
    \sigma&=[0.274, 0.276, 0.067, 0.131, 0.028, \\&0.008, 0.024, 0.002, 0.008 , 0.083, 0.574]
\end{align*}
Note $\pi$ denotes the mixture weights, $\mu$ the means and $\sigma$ the standard deviation. There are a total of 11 components and the resulting distribution was mapped through $\frac{1}{1+e^{-x}}$ to guarantee the action bounds. Figure \ref{fig:noise_openai} shows a sample from the distribution. The distribution was chosen to have more modes than others to offer more diversity in our experiments.

\subsection{Train \& Test Sets}
\label{apdx:train_vs_test}
For \textit{Wet Chicken}, \textit{Pendulum-v0} and \textit{Hopper-v2} replay buffers were gathered for both the Train and Test sets. In order to ensure diversity between both sets, the test set was gathered by trained SAC policy while the training set was gathered with a random policy. The SAC policy was trained from the following repository https://github.com/pranz24/pytorch-soft-actor-critic. For \textit{Hetero} and \textit{Bimodal} the training sets were gathered as described whereas the test sets were gathered uniformly at random.

\begin{table*}[t]
\centering
\begin{tabular}{l|l|l|l|l|l|l|l}
      Env  &   Acq Batch  & GP & PNEs & MC\_drop& Nflows & Nflows\_out & Nflows\_base \\
      \midrule
\multirow{4}{*}{\textit{hetero}} & 10  &  -2.08±0.02 & \textbf{1.01±0.09} &   \textbf{0.9±0.19}& -0.04±0.16 & \textbf{0.98±0.2} & \textbf{1.04±0.21} \\
          & 25  &  -2.06±0.01 & \textbf{1.07±0.08} & \textbf{0.97±0.14} &                0.25±0.4 &    \textbf{1.18±0.18} &   \textbf{1.11±0.12} \\
          & 50  &  -2.06±0.01 &   1.13±0.03 &  1.02±0.12 &               0.23±0.14 &    \textbf{1.33±0.11} &   \textbf{1.21±0.11} \\
          & 100 &  -2.06±0.01 &   \textbf{1.23±0.06} &    1.1±0.1 &               0.61±0.36 &    \textbf{1.32±0.09} &   \textbf{1.25±0.07}\\

       \midrule
\multirow{4}{*}{\textit{bimodal}} & 10  &  -4.97±0.04 &   0.62±0.04 &  0.58±0.05 &  -1.3e7.0±3.7e7 &   -1.82±7.76 &   \textbf{1.13±0.18} \\
          & 25  &   -5.0±0.03 &   0.64±0.03 &  0.61±0.03 &        -3.3e3±9.5e3 &    \textbf{1.37±0.04} &   \textbf{1.38±0.04} \\
          & 50  &   -5.0±0.06 &   0.67±0.03 &  0.62±0.03 &               0.42±0.46 &     \textbf{1.4±0.05} &   \textbf{1.45±0.04} \\
          & 100 &  -4.96±0.03 &   0.68±0.02 &  0.64±0.03 &                0.56±0.3 &    \textbf{1.46±0.02} &   \textbf{1.47±0.01}\\
       \midrule
\multirow{4}{*}{\textit{WetChicken}} & 10  &   0.17±0.19 &   \textbf{1.22±0.23} &  0.44±0.16 &  0.62±0.84 &  -7.3±6.98 & -0.77±3.12 \\
          & 25  &   0.26±0.06 &   \textbf{1.29±0.22} &  0.65±0.12 &             0.5±0.67 &   -0.11±3.03 &   \textbf{1.77±1.52} \\
          & 50  &   0.28±0.05 &   1.27±0.29 &  0.81±0.05 &               0.63±0.57 &     \textbf{1.9±0.36} &   \textbf{2.75±0.99} \\
          & 100 &   0.31±0.01 &    1.19±0.3 &   0.87±0.2 &               0.81±0.67 &    \textbf{1.86±1.36} &   \textbf{2.49±1.15} \\
       \midrule
\multirow{4}{*}{Pendulum-v0} & 10  &    3.05±0.0 &   \textbf{6.11±0.59} &  2.49±0.25 &           -79.21±106.47 &  -0.95±13.34 &   5.74±3.07 \\
          & 25  &    3.07±0.0 &    7.4±0.67 &  3.46±0.31 &             -42.66±73.4 &     \textbf{8.92±1.1} & \textbf{9.27±1.68} \\
          & 50  &    3.08±0.0 &    7.5±0.78 &  4.12±0.18 &               3.93±7.02 &    \textbf{9.78±2.56} & \textbf{10.57±0.92} \\
          & 100 &     3.1±0.0 &   7.72±0.72 &   4.58±0.1 &               7.86±1.54 &    \textbf{9.82±1.82} & \textbf{11.27±1.0} \\
          \midrule
\multirow{4}{*}{\textit{Hopper-v2}} & 10  &   \textbf{30.5±0.36} &  17.52±1.14 &  6.92±0.85 &            -33.35±28.65 &   19.94±4.18 &  14.01±2.32 \\
          & 25  &   \textbf{32.4±0.08} &  21.84±0.76 &  7.06±0.93 &           -69.32±139.11 &   \textbf{31.03±1.59} &  24.42±1.76 \\
          & 50  &   \textbf{33.0±0.07} &  22.32±1.66 &  7.66±1.11 &              -4.39±8.12 &   \textbf{31.87±5.13} &   28.1±0.91 \\
          & 100 &  \textbf{33.69±0.05} &  24.12±1.13 &  8.02±1.29 &                 5.4±4.7 &   \textbf{33.48±1.13} &  29.92±1.35 \\
       \bottomrule
\end{tabular}
\caption{Log Likelihood of a held out test set during training at different acquisition batches. Experiments were across ten different seeds and the results are expressed as mean plus minus one standard deviation.}
\label{tbl:likelihood}
\end{table*}
\begin{figure*}[ht!]
\centering
\includegraphics[width=0.9\textwidth]{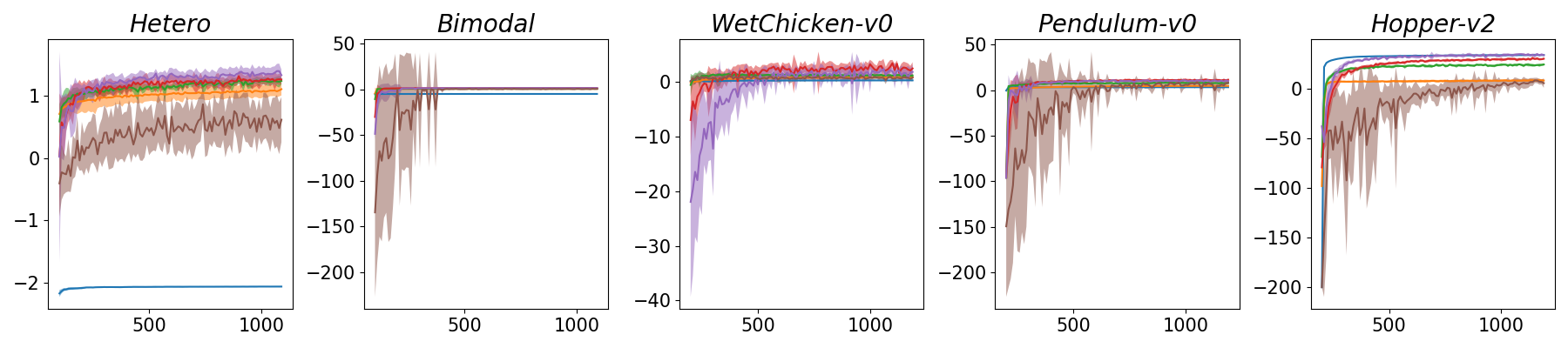} %
\begin{tabular}{c c c c c c}
\raisebox{0.7ex}{\colorbox{blue_gp}{ }} GP & \raisebox{0.7ex}{\colorbox{orange_mc_drop}{ }} MC Drop  & \raisebox{0.7ex}{\colorbox{green_pne}{ }} PNE & \raisebox{0.7ex}{\colorbox{brown_nflows}{ }} Nflows & \raisebox{0.7ex}{\colorbox{purple_nflows_out}{ }} Nflows Out & \raisebox{0.7ex}{\colorbox{red_nflows_base}{ }} Nflows Base
\end{tabular}

\caption{Mean Log Likelihood on 50 randomly sampled test set inputs as data was added to the training sets.}
\label{fig:actlearn_allmodels_log_likeli}
\end{figure*}
\subsection{Additional Results}
\label{apdx:more_results}
In addition to the KL metric reported in the text, we evaluated our method on RMSE and Log Likelihood. The RMSE is reported in Table \ref{tbl:likelihood} and the active learning curves in Figure \ref{fig:actlearn_allmodels_rmse}. In all cases Nflows Out and Nflows Base perform favorably to existing methods. Though, \textit{WetChicken} seems to be a particularly difficult environment for each model to reduce its RMSE as they acquire more data. A similar effect was seen for the KL divergence, for PNEs, MC Drop and GPs, and we believe this is because \textit{WetChicken} has the most interesting stochasticity, both heteroscedastic and multimodal, and thus is difficult to model.

The Log Likelihood results are reported in Table \ref{tbl:likelihood} and the active learning curves in Figure \ref{fig:actlearn_allmodels_log_likeli}. As before Nflows Out and Nflows Base perform favorably to the other baseline models except for the \textit{Hopper-v2} environment. We believe that GPs perform as well in this setting as the noise is homoscedastic. Despite that it is not Gaussian noise, a Gaussian approximation performs well enough.

We thought it good practice to include RMSE and Log Likelihood evaluations as they are common. Though, we found KL Divergence the most apt evaluation metric as we are most interested in distributional fit. We want to accurately capture modes to be able to best capture rare events for safety reasons. Nflows Out and Nflows Base perform favorably across all three metrics. 

\subsection{Acquisition Criteria}
\label{apdx:acq_crit}
In addition to validating our models against baselines, we also evaluated different acquisition criteria. Figure \ref{fig:actlearn_nflows_out} shows the different learning curves across the three metrics on the \textit{Bimodal} environment for Nflows Out. Using the epistemic uncertainty outperforms all other acquisition criteria. Note that aleatoric uncertainty seems to not improve our model by sampling new data points. This is the case because there is high aleatoric uncertainty in the train set and thus data gets data from the same region. 

\begin{figure*}[ht!]
\centering
\includegraphics[width=0.9\textwidth]{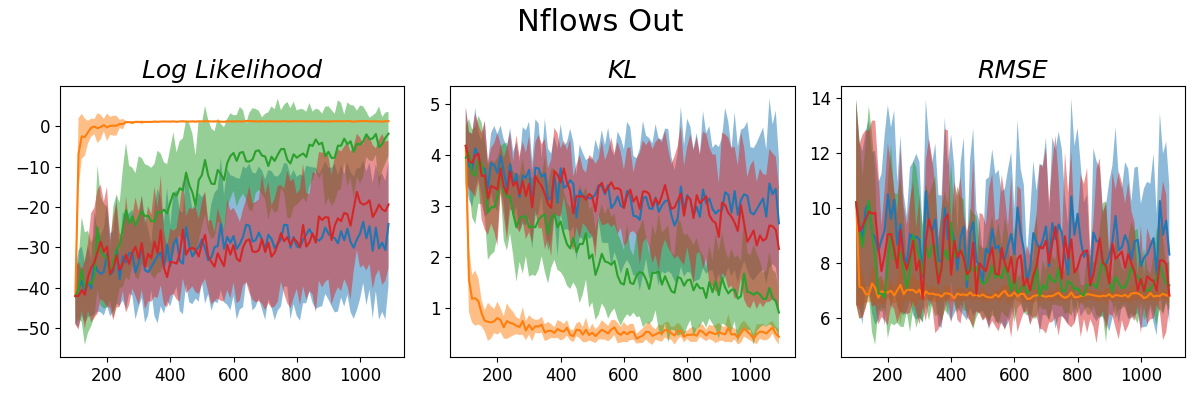} %
\begin{tabular}{c c c c}
\raisebox{0.7ex}{\colorbox{blue_gp}{ }} Aleatoric Unc. & \raisebox{0.7ex}{\colorbox{orange_mc_drop}{ }} Epistemic Unc.  &  \raisebox{0.7ex}{\colorbox{red_nflows_base}{ }} Total Unc. &
\raisebox{0.7ex}{\colorbox{green_pne}{ }} Random 
\end{tabular}

\caption{Active Learning curves for each metric reported for the Bimodal Environment with Nflows Out. Each color represents a different acquisition criteria.}
\label{fig:actlearn_nflows_out}
\end{figure*}

Figure \ref{fig:actlearn_nflows_base} depicts the same learning curves for Nflows Base. In this set of graphs, we can see that epistemic uncertainty calculated in the base distribution, Epistemic Unc. Base, and the output distribution, Epistemic Unc. Out, perform very similarly. Thus validating the use of the base distribution to estimate uncertainty. Both figures show that random accumulation of data is superior to that of aleatoric and total uncertainty.

\subsection{Monte Carlo High Dimensions}
\label{apdx:mc_high}
One can think of Monte Carlo estimation as randomly throwing darts at a box with a known area $A$ and multiplying $A$ by the proportion of darts that lie under the function of interest. In higher dimensions, we can think of this as the proportion of darts that land in the hypersphere when randomly throwing at the hypercube. As the number of dimensions $d$ increases the ratio of the volume of the hypersphere to the volume of the hypercube approaches zero.
\begin{proof}
We wish to show that $\lim_{d \to \infty}\frac{V(S)}{V(C)}=0$, where $V(S)$ and $V(C)$ are the volumes of a hypersphere and hypercube respectively. Let $R$ denote the length of the radius and side length of the hypersphere and hypercube. Note that we are choosing the largest such hypersphere that can fit into our hypercube. Thus,
\begin{align}
    \lim_{d \to \infty}\frac{V(S)}{V(C)} &= \lim_{d \to \infty}\frac{\frac{R^d\pi^{\frac{d}{2}}}{\Gamma(\frac{d}{2}+1)}}{R^d} \nonumber\\
    &=\lim_{d \to \infty}\frac{\pi^{\frac{d}{2}}}{\Gamma(\frac{d}{2}+1)} \nonumber\\
    &=\lim_{d \to \infty}\frac{\pi^{\frac{d}{2}}}{\frac{d}{2}\Gamma(\frac{d}{2})} \label{eq:prop_gamma}\\
    &=0\label{eq:gamma_faster}
\end{align}
One can get to line (\ref{eq:prop_gamma}) via the properties of the gamma function and to line (\ref{eq:gamma_faster}) by the fact that $\lim_{d \to \infty}\Gamma(d)>>a^d$, where $a$ is a constant. 
\end{proof}
\noindent Therefore in higher dimensions one requires more samples to get an accurate estimate as it will be harder to throw a dart into the hypersphere. Figure \ref{fig:monte_carlo} captures Monte Carlo's error in higher dimensions. As the the number of dimensions increase so to does the error between the Monte Carlo estimated entropy and the true entropy. In addition, we show the time savings of Nflows Base over Nflows Out. This can be seen in Table \ref{tbl:time_savings}.
 \begin{figure*}[ht!]
\centering
\includegraphics[width=0.9\textwidth]{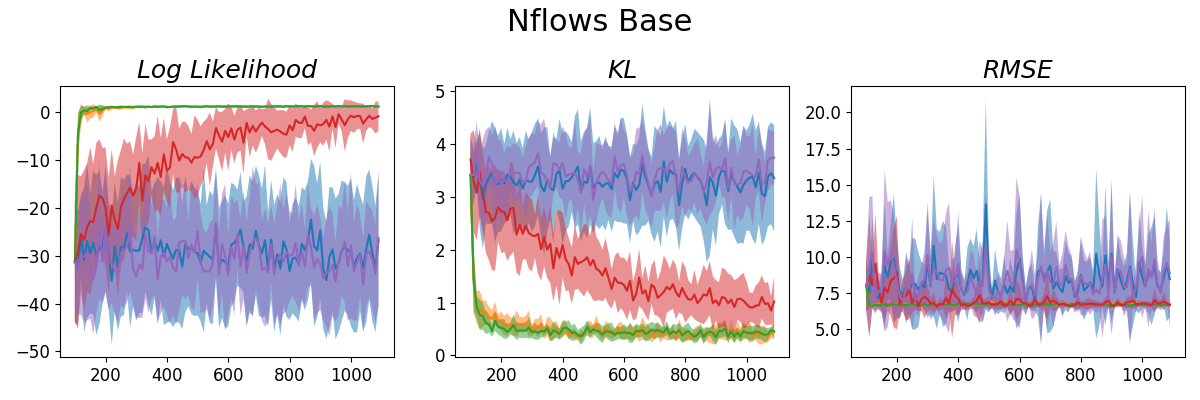} %
\begin{tabular}{c c c c c}
\raisebox{0.7ex}{\colorbox{blue_gp}{ }} Aleatoric Unc. & \raisebox{0.7ex}{\colorbox{orange_mc_drop}{ }} Epistemic Unc. Base & \raisebox{0.7ex}{\colorbox{green_pne}{ }} Epistemic Unc. Out & 
\raisebox{0.7ex}{\colorbox{purple_nflows_out}{ }} Total Unc. &
\raisebox{0.7ex}{\colorbox{red_nflows_base}{ }} Random 
\end{tabular}

\caption{Active Learning curves for each metric reported for the Bimodal Environment with Nflows Base. Each color represents a different acquisition criteria.}
\label{fig:actlearn_nflows_base}
\end{figure*}
 \begin{figure}[ht]
\centering
\includegraphics[width=0.9\columnwidth]{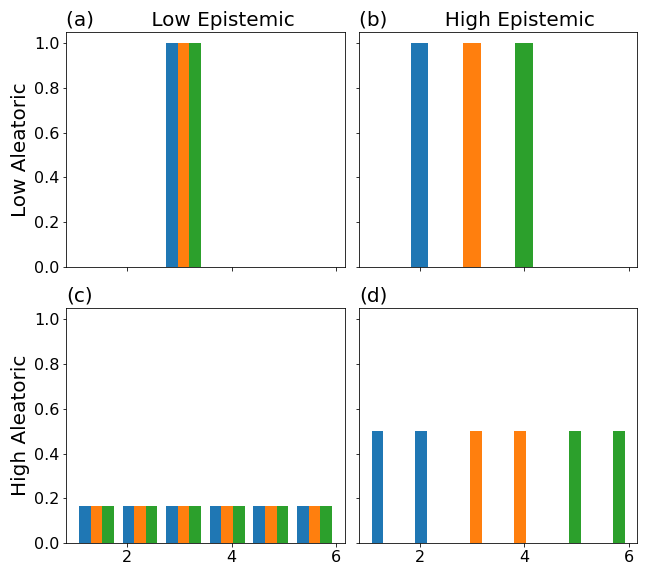} %
\caption{Each color represents another model's PMF for 4 different weighted dice experiments.}
\label{fig:epi_intuition}
\end{figure}
\subsection{Aleatoric vs Epistemic Uncertainty}
\label{apdx:intuition_alea_epi}
If one were to roll a fair dice, the probability mass function (PMF) is known and thus all uncertainty is known. This example represents aleatoric uncertainty. A dice with an unknown weight, the PMF is unknown but one could be estimated via samples. This is an example of epistemic uncertainty. Figure \ref{fig:epi_intuition} shows the four scenarios of uncertainty for different weighted dice experiments. Each graph depicts three models represented by different colors:
\begin{figure}[t]
\centering
\includegraphics[width=0.9\columnwidth]{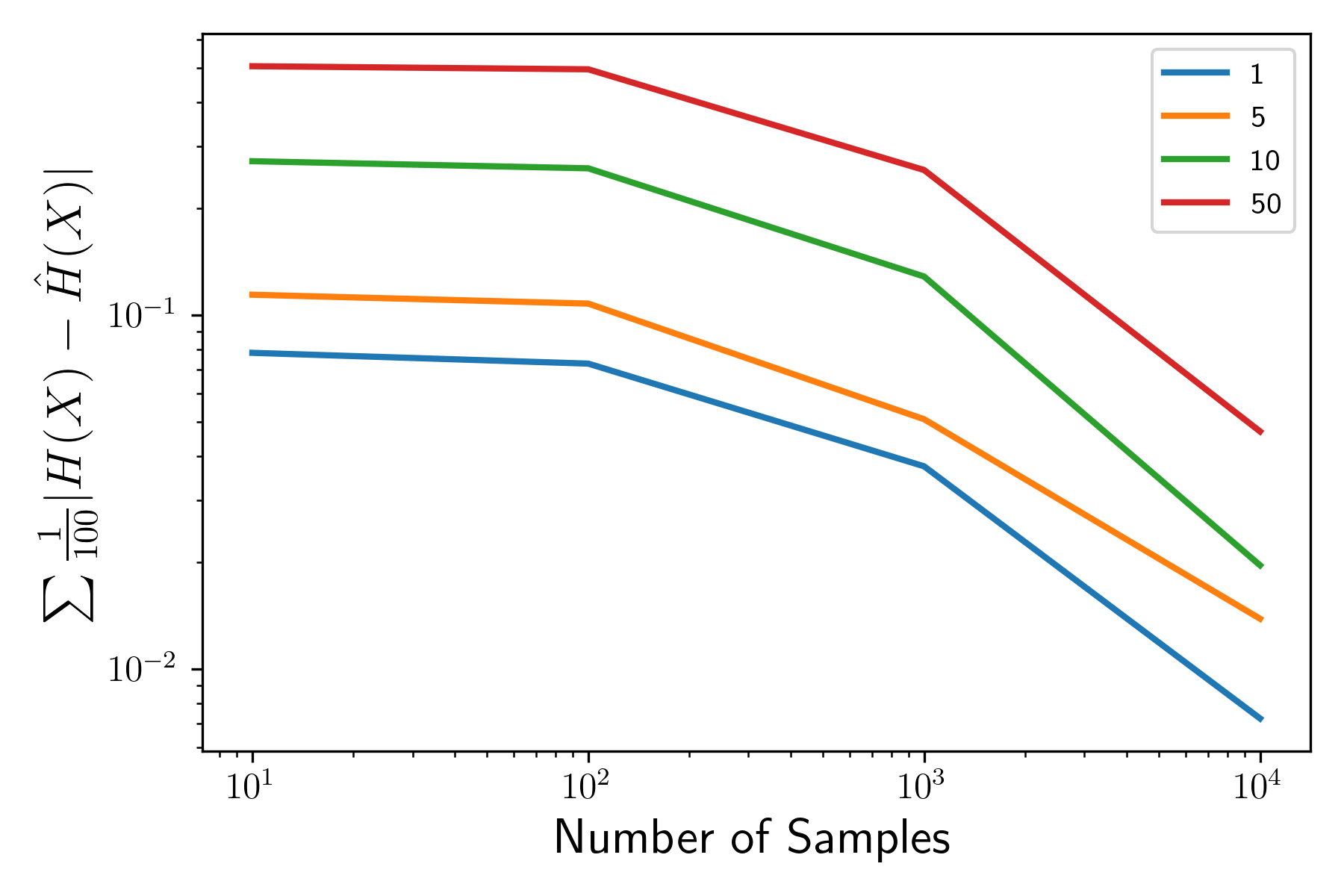} %
\caption{Error in estimating entropy for Nflows Base using a sampling based method $\hat{H}(X)$ vs. an analytical method $H(X)$ with increasing dimensions on a Log-Log scale. Note the results were run across 100 seeds and averaged.}
\label{fig:monte_carlo}
\end{figure}

\begin{table}[ht!]
\centering
\begin{tabular}{l|l}
      Env  &   Time Savings \\
      \midrule
\textit{hetero}& 44.64±1.18\\
       \midrule
\textit{bimodal} & 44.45±0.84 \\
       \midrule
\textit{WetChicken} & 49.00±5.42 \\
       \midrule
\textit{Pendulum-v0} & 64.77±5.06 \\
          \midrule
\textit{Hopper-v2} & 161.11±6.59 \\ 
       \bottomrule
\end{tabular}
\caption{Time savings in seconds of Nflows Base vs Nflows Out run over 10 seeds for one acquisition of uncertainty estimation. Mean plus minus one standard deviation reported.}
\label{tbl:time_savings}
\end{table}

 \begin{itemize}
     \item[(a)] Shows an instance of low epistemic and low aleatoric. All three models agree and their PMFs put their mass on one outcome.
     \item[(b)] Depicts an instance of high epistemic and low aleatoric. All three models disagree and their PMFs put their mass on one outcome (2, 3, or 4).
     \item[(c)] Displays an instance of low epistemic and high aleatoric. All three models agree and their PMFs put their mass equally on each outcome.
     \item[(d)] Shows an instance of high epistemic and high aleatoric. All three models disagree and their PMFs put their masses across all possible outcomes.
 \end{itemize}
Thus in order to collect data efficiently, one would most want to gather data for graphs (b) and (d) instead of (a) and (c). Despite the fact that graph (c) has high entropy, it is unlikely that more data could improve the models as they agree.

\subsection{MI Invariant to Diffeomorphisms}
\label{apdx:MI_invariance}
\begin{theorem}
\label{thm:PaiDE}
Let $y=g(b)$ where $g$ is a diffeomorphism, i.e. it is a bijection which is differentiable with a differentiable inverse. Then $I(Y,W)=I(B,W)$ where $W$ is a random variable representing ensemble components.
\end{theorem}
\begin{proof}
    Using the change of variable formula in probability yields:
    \begin{align*}
    p_{Y}(y)&=p_{B}(g^{-1}(y))|\det(J_{g^{-1}}(y))| \\
    p_{Y,W}(y,w)&=p_{B,W}(g^{-1}(y),W)|\det(J_{g^{-1}}(y))|
    \end{align*}
    Applying these definitions,
    \begin{align*}
    I(Y,W)&=\sum_w\int_yp_{Y,W}(y,w)\log\frac{p_{Y,W}(y,w)}{p_{Y}(y)p_{W}(w)}dy\\
    &=\sum_w\int_bp_{B,W}(b,w)\log\frac{p_{B,W}(b,w)|\det(J_{g^{-1}}(y))|}{p_{Y}(y)p_{W}(w)}db\\
    &=\sum_w\int_bp_{B,W}(b,w)\log\frac{p_{B,W}(b,w)}{p_{B}(b)p_{W}(w)}db\\
    &=I(B,W).
    \end{align*}
\end{proof}

\end{document}